%% file: main.tex
\pdfoutput=1
\documentclass[nohyperref]{article}
\usepackage{acl}

\usepackage{times}
\usepackage{latexsym}
\usepackage[T1]{fontenc}
\usepackage[utf8]{inputenc}
\usepackage{microtype}
\usepackage{graphicx}
\usepackage{booktabs}
\usepackage{hyperref}

\input{./math_commands.tex}

\input{./our_math.tex}

\usepackage{subcaption}
\usepackage{booktabs}
\usepackage{multirow}
\usepackage{float}
\usepackage[symbol]{footmisc}
\usepackage{tikz-network}
\usepackage{graphicx}
\graphicspath{{./figures/}}

\renewcommand{\Secref}[1]{Sec.~\ref{#1}}
\renewcommand{\secref}[1]{Sec.~\ref{#1}}
\renewcommand{\Figref}[1]{Fig.~\ref{#1}}
\renewcommand{\figref}[1]{Fig.~\ref{#1}}

\newcommand{\method}[0]{TCUP}

\title{Taming Continuous Posteriors for Latent Variational Dialogue Policies}

\author{Marin Vlastelica\textsuperscript{1, 2}, Patrick Ernst\textsuperscript{2}
\and György Szarvas\textsuperscript{2} \\
    \textsuperscript{1}Autonomous Learning Group, Max Planck Institute for Intelligent Systems, T\"ubingen, Germany\\
    \textsuperscript{2}Amazon Development Center Germany GmbH, Berlin, Germany\\
    \textsuperscript{1}marin.vlastelica@tuebingen.mpg.de, \textsuperscript{2}\{peernst, szarvasg\}@amazon.de}

\begin{document}

\maketitle

\begin{abstract}
  Utilizing amortized variational inference for latent-action reinforcement learning (RL) has been shown to be an effective approach in Task-oriented Dialogue (ToD) systems for optimizing dialogue success.
  Until now, categorical posteriors have been argued to be one of the main drivers of performance.

  In this work we revisit Gaussian variational posteriors for latent-action RL and show that they can yield even better performance than categoricals.
  We achieve this by simplifying the training procedure and propose ways to regularize the latent dialogue policy to retain good response coherence.
  Using continuous latent representations our model achieves state of the art dialogue success rate on the MultiWOZ benchmark, and also compares well to categorical latent methods in response coherence.
\end{abstract}

\section{Introduction}

Task-oriented Dialogue (ToD) systems have reached a degree of maturity, which enables them to engage with human users and assist them in various tasks.
They are able to steer natural-language conversations in order to complete users' goals, such as booking restaurants, querying weather forecasts and resolving customer service issues.
At their core, the behavior of these systems is controlled by a dialogue policy, which receives user inputs in the form of utterances and additional features or states.

Template-based methods~\cite{walker2007individual, inaba2016neural} leverage ranking or classification approaches to select the most fitting response from a pre-defined set of responses, i.e. templates.
While template-based methods offer better control over the dialogue policy behavior,
they are less versatile due to their dependency on template sets. 
Moreover, constructing comprehensive template sets is a challenge in itself \cite{Gao:2019}.
In retrieval-based approaches~\cite{yan2016learning, henderson2019training, tao2019multi} candidate responses
are not predefined, but are retrieved from massive dialogue corpora, e.g. by executing ad-hoc search queries a priori.

Generative models do not require such additional inputs as prior knowledge.
They enable end-to-end (E2E) learning of dialogue
policies with the potential of extrapolation to diverse responses \cite{serban2017hierarchical, zhao2017learning, gu2018dialogwae}.
Methods based on amortized variational inference achieve this by inferring low-dimensional variational posteriors mapped from dialogue contexts, similar to variational autoencoders (VAEs)~\citep{kingma2013auto}.
Even though such fully data-driven approaches offer great versatility
and a faster adoption,
they may exhibit degenerate behavior by generating incomprehensible utterances.
This is apparent in multi-turn dialogues, which span hundreds of words, while the
success signal is only observed at the end of dialogues.
Reinforcement learning (RL) based approaches are able to optimize for such long-term, sparse rewards and have been applied in this setup.
Previous approaches prominently applied word-level RL \cite{lewis2017deal, kottur2017natural},
 where the action space is defined over the entire vocabulary.
 The response utterances are then generated auto-regressively by consecutive next-word predictions.
Unfortunately, the use of large action spaces often impedes the convergence of policy
learning algorithms, which makes it hard to ensure coherent responses.
Prior work makes use of \emph{latent-action RL} to address the dimensionality problem by utilizing variational inference approaches \citep{zhao:2019,lubis:2020}.
These methods rely on a supervised learning stage followed by fine-tuning via reinforcement learning in the latent space.

In this paper we follow this paradigm, but extend prior work substantially by introducing the \method{} approach, which aims to \emph{T}ame \emph{C}ontin\emph{U}ous \emph{P}osteriors for latent variational dialogue policies. \method{} makes the following contributions:
\textbf{(i)} A new formulation of the variational inference objective for
learning continuous latent response representations without auxiliary learning objectives.
\textbf{(ii)} A more robust approach for learning from offline ToD data in a
RL setup, which utilizes the fact that we are dealing with
expert dialogue trajectories.

We compare the performance of our proposed method to competing approaches on
the MultiWOZ benchmark~\citep{budzianowski2018MultiWOZ}.
Our experimental results show that we are able to improve
the state-of-the-art performance across different benchmark metrics.
Apart from MultiWOZ's context-to-text metrics, we demonstrate the benefits of
\method{}'s learned latent representations quantitatively using a clustering analysis following \citet{lubis:2020}. %

\section{Preliminaries}\label{sec:preliminaries}
Usually, latent-action reinforcement learning methods are trained in two stages:
\textbf{(i)} In the first stage an encoder-decoder architecture
is applied to learn latent representations from a supervised signal, i.e. latent actions, over
dialogue responses.
\textbf{(ii)} In the consecutive stage a RL-based policy predicts the best latent
action given a particular context in order to optimize for long-term
dialogue success.
The trained decoder stays fixed and receives the output of the RL policy to
generate the final reponse utterance.

\paragraph{Learning Latent Response Representations}
We denote the dialogue context as $\rc$, which contains the user input utterance
and dialogue state, and the response utterance generated by the system as $\rx$.
Provided a dataset of context and optimal response pairs ${(\rc, \rx)}$,
we want to extract a latent representation $\rz$, representing the
dialogue responses given context.
Approaches based on variational inference have shown to be beneficial for learning such latent representations.
This is done by optimizing the evidence lower bound (ELBO):
\begin{equation} \label{eq:full-elbo}
  \begin{split}
      \gL(\phi, \theta) = \E_{q_{\theta}( \rz| \rx, \rc)}  [-\log p_\phi(\rx | \rz) ] \\
      + \KL[ q_\theta( \rz | \rx, \rc) || p( \rz | \rc )]
  \end{split}
\end{equation}
Where $q_{\theta}$ denotes the variational posterior parametrized by $\theta$ and  $p_\phi$ the decoder parametrized by $\phi$.
Arguing that the full ELBO formulation suffers from ``explaining away'', \citet{zhao:2019} introduced the ``lite'' ELBO to mitigate this issue:
\begin{equation} \label{eq:lite-elbo}
  \begin{split}
      \gL(\phi, \theta) = \E_{q_\theta( \rz| \rc)}  [-\log p_\phi(\rx | \rz) ] \\
      + \KL[ q_\theta( \rz | \rc) || p( \rz )]
  \end{split}
\end{equation}
The goal is to be closer to the information available during testing time,
where the decoder only sees $z$ conditionally sampled on the context $q_\theta(\rz | \rc)$.
\paragraph{Learning Latent Dialogue Policies}
After learning to extract a compressed representation $\rz$ in the supervised learning stage, the encoder $q_\phi(\rz | \rc)$ is fine-tuned via reinforcement learning to optimize the dialogue reward,
where mostly directly  the success metric is used.
A Markov Decision Process ($\mdp$) is defined as a tuple ($\mathcal{S}, \mathcal{A}, r, p$) of state space $\mathcal{S}$, action space $\mathcal{A}$, reward function $r$, and transition density $p$.
The general goal of reinforcement learning is to optimize the expected return of policy $\pi$, denoted as $\E[J(\pi)]$.
Many works utilize a Monte-Carlo estimate of the policy gradient:
\begin{equation} \label{eq:mc-pgrad}
  \nabla_\phi \E  \Big [J(\pi_\phi) \Big ] = \E \Big [  \nabla_\phi \sum_{(\rs, \ra) \in \tau} \log \pi_\phi(\ra | \rs) r(\rs,\ra) \Big ]
\end{equation}
Where the expectation is taken over the trajectory distribution $\eta^\pi(\tau)$, i.e. distribution of sequences of $(\rs, \ra)$.
In the context of ToD, we may cast the state $\rs$ as being the underlying dialogue state which may be unobserved, in which case the problem becomes partially observable.
In our setting, the latent dialogue policy is warm-started by the parameters of the variational posterior from the first stage of training.

\section{Method}\label{sec:method}
\begin{figure*}
  \centering
  \includegraphics[width=0.9\linewidth]{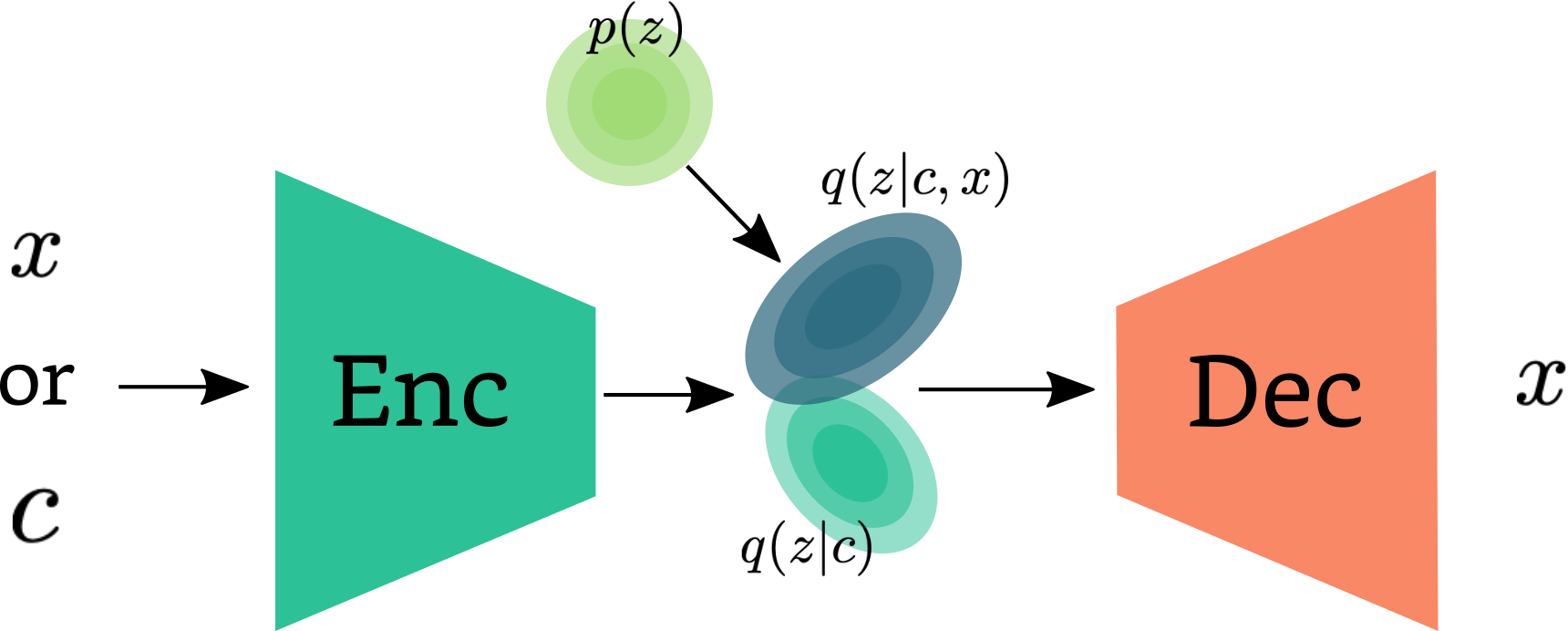}
  \caption{
    High-level schema of the proposed supervised training stage.
    The encoder receives either context or both the context and the response that is to be decoded from latent $\rz$.
    In a sense, when samples from $q(\rz | \rc, \rx)$ are used for decoding,
    the architecture acts as a proper conditional autoencoder.
    We use the normal distribution family for modelling the distribution of $\rz$.
    }
\end{figure*}
In line with the general latent action RL paradigm
outlined in \Secref{sec:preliminaries}, \method{} is also
based on a two-stage approach.
In \secref{sec:elbo} we describe the reformulation of the ELBO
used in the supervised learning stage, which is vital to reach
state-of-the performance with continuous variational posteriors.
In \secref{sec:coh} we discuss our reinforcement learning setup
and introduce methods to address the deterioration of \emph{latent action} based dialogue policies by regularization.

\subsection{Revisiting the Full ELBO}\label{sec:elbo}
As aforementioned prior work by \citet{zhao:2019} has introduced the ``lite'' ELBO
(described in \eqref{eq:lite-elbo}) to alleviate the overexposure bias that emerges when they incorporate response information in the optimization objective.
Indeed, if we only optimize for samples of $q(\rz | \rc, \rx)$,
the most information about the responses ($\rx$) is found within the responses themselves.
However, not conditioning on $\rx$ will reduce the expressiveness of
the variational posterior, especially if responses and contexts share indicative
patterns.

\def\qp{{q^p_\zeta}}
\def\q{{q_\theta}}
\def\p{{p_\phi}}
In this work we rely on a full ELBO and introduce a conditional prior $\qp(\rz| \rc)$ indicated by the superscript $p$
which is constrained by the free prior $p(z)$ and also parametrized.
This leads us to the following objective:
\begin{equation} \label{eq:full-elbo-main}
  \begin{split}
    \gL(\phi, \theta, \zeta) &= \KL[ q_\theta || \qp] + \KL[\qp || p(\rz)] - \\
    &\E_\q[\log \p(\rx | \rz )] - \E_\qp[\log \p(\rx | \rz)]
  \end{split}
\end{equation}
We defer the full derivation of this objective to the Appendix \Secref{sec:app:math}.

\noindent{}During training time, we leverage samples from both the prior and the variational
posterior for reconstruction, which makes sure that the decoder doesn't overly
rely on information from $\rx$.
Apart from achieving better performance as outlined in \Secref{sec:experiments}, this also simplifies the training procedure considerably.

\subsection{Regularizing for Success and Coherence}\label{sec:coh}
Reinforcement learning algorithms have shown to be able to exploit weaknesses
in simulators of game environments \citep{mnih2013playing} by finding high
reward states which don't align with solving the underlying task.
Similarly, in the context of latent action RL for ToD,
the decoder can be seen as a weak simulator that is prone to exploitation due to the nature of the success rate metric.
By checking for relevant slots to be present in the response, success rate permits that long, possibly incoherent responses more easily achieve success due to higher chance of containing the correct slot tokens.
Anecdotal evidence for this is presented in Table~\ref{tb:dialogue}.
When using continuous latent actions, this issue is even more severe, since
 the model is able to sample out of distribution (OOD) utterances that provide success.
For measuring response coherence, we rely on the BLEU metric as specified in the MultiWOZ benchmark.

\subsubsection*{Penalizing Out of Distribution Samples}
In the supervised training stage we apply variational inference and
implicitly maximize the BLEU score through maximizing the ELBO(see \eqref{eq:mc-pgrad}).
We denote this BLEU-maximizing policy with $\pi^{VI}$.
Since we constrain our policy with an
isotropic Gaussian prior in the first training stage, we can leverage this
information to prevent the policy from deviating from this prior in the form of a
divergence cost which can be efficiently computed.
Concretely, the regularized reward function is defined as follows:
\begin{equation}
  \begin{split}
    r(\rx, \rc) = succ(\rx, \rc) - \beta \KL[\pi(\rz | \rc) \,|| \, p(\rz)]
  \end{split}
\end{equation}

\subsubsection*{ToD as Offline RL}
We are dealing with an offline reinforcement learning problem,
since we have a dataset of optimal responses without the possiblity of
obtaining more samples via a simulator or users.
By shifting the reinforcement learning problem to the latent space,
we are implicitly creating a surrogate online problem, where
we need to obtain samples form $\pi(\rz | \rc)$ and evaluate them.

We argue that one of the reasons why Gausssian latent spaces have been reported
as under-performing in comparison to categoricals is the biased and noisy gradient
estimate based on samples from a single dialogue.
Contrary to prior work \cite{lubis:2020,zhao:2019}, which estimates the policy
gradient over the responses from a single dialogue sample,
we take advantage of a Monte Carlo policy gradient estimate across \emph{multiple dialogues}.

\subsubsection*{Replaying Succesful Samples}

Re-using encountered experience by storing it in memory (also called a replay buffer) has proved to be beneficial for sample-efficiency in
reinforcement learning.
However, a naive usage of the buffer has multiple caveats in the MultiWOZ setting.
Firstly, since the success signal is calculated on the dialogue level, some
responses that might be actually successful conditioned on the dialogue state
and context, might be labelled as negative.
Intuitively, the policy can start off the dialogue correctly, but fail to successfully complete it.
This can lead to conflicting examples in the replay buffer, which destabilises training.
Secondly, we have a many-to-one mapping from responses to success, which leads
to multimodality, but also modes that might be incoherent.

If we make the observation that the response at time step $t$ is conditionally independent of the dialogue history given the dialogue state and input utterance,
we notice that we can attribute success directly to the response independent of past utterances.
This motivates the storage of only successful responses in the replay buffer, which we sample
 as a fraction of the training batch.
That way it mitigates the problem of false negative responses and ensures fewer conflicting examples in the batch.
Replaying past experience also ensures that certain difficult samples are not forgotten, which both increases training stability and ensures that the current policy $\pi$ stays close to $\pi^{VI}$.

In practice, we exchange a certain sample in the batch generated
by the current policy with a sample from the replay buffer with probability $\lambda$.
As $\lambda \rightarrow 0$ we arrive to the simple REINFORCE update,
$\lambda=1$ means that we only use replayed samples for updates.

\begin{table*}[!th]
  \centering
  \caption{Results comparison to competing methods on MultiWOZ.
  We have marked methods that utilize  RL with $\dagger$, transformer architecture with $\star$.
  }
  \resizebox{\textwidth}{!}{
    \begin{tabular}{ l | c | c | c | c | c | c | c}
      \toprule
      Model  & BLEU	& Inform	& Success	 & Av. len. &	CBE	 & \# unigrams	& \# trigrams \\
      \midrule
      MarCo$^\star$\ifdefined\tablecites \cite{wang2020multi} \fi  &   	17.3	 & 94.5 & 	87.2	 & 16.01	 & 1.94	 & 319	& 3002 \\
      HDSA\ifdefined\tablecites \cite{chen2019semantically} \fi    & \textbf{20.7} & 	87.9	 & 79.4 & 	14.42 & 	1.64	& 259 & 	2019 \\
      HDNO$^\dagger$\ifdefined\tablecites\cite{wang2020modelling} \fi      &	17.8 &	93.3 &	83.4 &	14.96 &	0.84	& 103 &	315 \\
      SFN$^\dagger$\ifdefined\tablecites\cite{mehri2019structured}  \fi   &	14.1 &	93.4	 & 82.3 &	14.93 &	1.63 &	188 &	1218 \\
      UniConv\ifdefined\tablecites\cite{le2020uniconv}  \fi &   	18.1 &	66.7 &	58.7 &	14.17 &	1.79 &	338 &	2932 \\
      LAVA$^{\dagger}$\ifdefined\tablecites\cite{lubis:2020}  \fi \footnotemark[1]   &	16.2  &	89.7 &	77.6 &	14.41 &	1.96 &	272 &	2365 \\
      LAVA$^\dagger$-Cat\ifdefined\tablecites\cite{lubis:2020} \fi    &	10.8 &	95.9 &	93.5 &	13.28 &	1.27 &	176 &	708 \\
      \midrule
      \method{}$^\dagger$  & 10.3 &  \textbf{96.3} & \textbf{95.9} & 15.14 & 1.44 & 210 & 1838 \\
      \method{}$^\dagger$-Cat   & 14.3 & 96.1 & 92.1 & 14.33 & 1.61 & 230 & 1490 \\
      \bottomrule
    \end{tabular}
  }

  \label{tab:results}
\end{table*}

\begin{table*}[!th]
  \centering
  \caption{Example of successfull responses where the policy was able to
  ``game'' the metric, leading to a successful dialogue but resulting in incoherent utterances and hence low BLEU score.}
  \label{tb:dialogue}
  \begin{tabular}{ | l | p{12cm} |}
    \hline
    USER &  we are staying [value\_count] people for [value\_count] nights starting from [value\_day] . i need the reference number?\\
    \hline
    SYSTEM & [hotel\_name] [hotel\_name] [hotel\_name] [hotel\_reference] [hotel\_name] [hotel\_reference] [hotel\_name] [hotel\_reference] [hotel\_name] [hotel\_phone] [hotel\_phone] [hotel\_phone] [hotel\_phone] [hotel\_phone] [hotel\_phone]   [hotel\_phone] [hotel\_phone] [hotel\_phone] [hotel\_name] [hotel\_phone] [hotel\_name] -ly [hotel\_phone] [hotel\_name] -ly [hotel\_name] -ly [hotel\_name] -ly [hotel\_name] -ly [hotel\_name] -ly [hotel\_name]  ... \\
    \hline
    USER & just the address please \\
    \hline
    SYSTEM & [attraction\_name] is located at [attraction\_address], postcode [attraction\_postcode]. the phone number is [attraction\_phone]. [attraction\_name] is located at [attraction\_address], [attraction\_postcode]. phone number is [attraction\_phone] [attraction\_name] is located at [attraction\_address], [attraction\_postcode]. phone number is [attraction\_phone]. [attraction\_name] is located at [attraction\_address], [attraction\_postcode]. phone\\
    \hline
  \end{tabular}
\end{table*}

\section{Experiments}\label{sec:experiments}
We provide a detailed evaluation of \method{}'s dialogue policy in \Secref{sec:diapolresults}.
This includes performance on the MultiWOZ benchmark; an ablation study to assess the importance of our technical contributions proposed in \Secref{sec:method}; and a qualitative analysis of response coherence.
Finally, we analyze the quality of the latent representations in \Secref{sec:latentanalysis} and provide evidence that \method{} is able to learn representations using continuous latent variables which yield good separation and clustering of domains and actions present in MultiWOZ.
We make use of a recurrent encoder-decoder architecture with dot-product attention, further details on the architecture and training procedure can be found in Appendix \Secref{sec:app:modeltraining}.

\paragraph{Supervised Learning Stage (SL)} Here, we learn the mapping from context to response
using \eqref{eq:full-elbo}. In practice, we optimize the $p_\phi(\rx | \rz)$ part with a weighted cross-entropy loss that puts higher weights on slot placeholders in the response.
The best model is selected based on its BLEU score.

\paragraph{Reinforcement Learning Stage (RL)} In this stage we fix the parameters of the decoder and train only the encoder parameters via policy gradient with two important modifications: (i) we use a batched version of the policy gradient which contains samples from multiple dialogues and hence reduces the variance of the gradient, and (ii) in each batch we sample a mix of newly generated and old experience (with which we have obtained a success signal earlier). This
implicitly keeps us close to the starting policy which results in more stable training.

We have found that it is also beneficial to replace the standard  $\KL[q || p]$ term of the variational objective with a symmetric version of it $\frac{1}{2} (\KL[q || p] + \KL[p || q])$.
This ensures that regions where the densities of $p$ and $q$ behave differently are treated equally irrespective of the ordering.

\subsection{Context-to-Response Generation}\label{sec:diapolresults}
We validate the proposed method on the MultiWOZ benchmark \citep{wang2020multi}
in the policy learning task with ground-truth dialoge states and use the
same delexicalization approach as in \citet{lubis:2020}.
Table~\ref{tab:results} shows that \method{} improves the state-of-the-art inform- and success rate metric across all competitors.
Also, it is competitive in terms of language diversity metrics.
Compared to the currently best performing latent action reinforcement learning approach (LAVA), we increase all metrics except for minor decreases in the BLEU metric.
The response coherence score is further discussed in \Secref{sec:eval_bleu_deterioration}.

\paragraph{Latent Representations}
While prior latent-action reinforcement learning approaches\citep{lubis:2020,zhao:2019} favor categorical latent
distributions with modified attention mechanisms in the decoder,
our results demonstrate that relying on isotropic Gaussian latents
is advantageous, even if we only use simple dot-product attention in the decoder.
In particular, if we base \method{} on
categorical latent distributions (coined \method{}-Categorical in
Table~\ref{tab:results}) we observe competitive results
in terms of inform and success rate compared to prior work, but inferior
performance compared to \method{} with continuous latents.

Using categoricals limits us in the sense that we don't make full use of the nonlinearity in the decoder to decode diverse responses, which manifests itself with poor results in the diversity metrics (Table \ref{tab:results}).
By using a continuous latent distribution, we are able to improve the diversity of the generated responses compared to those two competitors.

\footnotetext[1]{Results taken for best runs, mean performance is actually lower, additional commentary available in Appendix \secref{app:sec:lava-gauss}}
\paragraph{Ablation Study}
\Figref{fig:performance-hist} presents an ablation study of different parts of \method{}.
In general, optimal replay (\emph{rep}) and the constrained rewards with the additional KL penalty (\emph{KL}) lead to higher success rate and BLEU, with optimal replay being the superior choice between the two, especially in the BLEU metric.
Further gains can be achieved by utilizing both.
Not using optimal replay introduced higher variance across runs.
During the first training stage, utilizing prior and
posterior shuffling (\emph{sl shuff}) enables a better fit in terms of BLEU scores.
Histograms of further metrics used in the MultiWOZ benchmark can be found in Appendix \Secref{app:histograms}.
In summary, these results (\emph{shuff kl rep}) validate the benefits of the proposed variational inference objective and regularization.

\begin{figure}[h]
    \centering
    \newlength{\subfig}
    \setlength{\subfig}{{1.0\linewidth}}
    \begin{subfigure}{\subfig}
        \includegraphics[width=\linewidth]{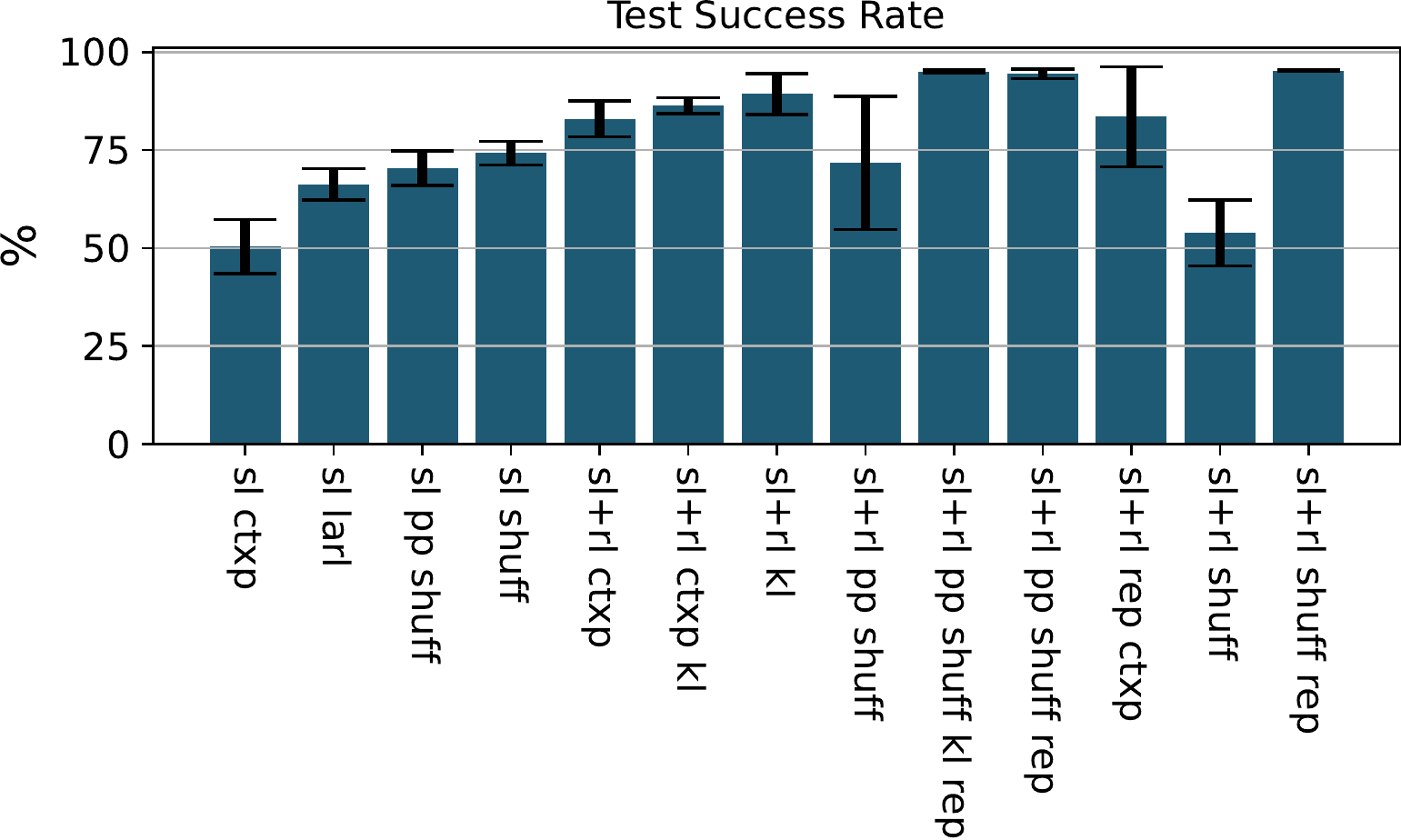}
    \end{subfigure}
    \begin{subfigure}{\subfig}
        \includegraphics[width=\linewidth]{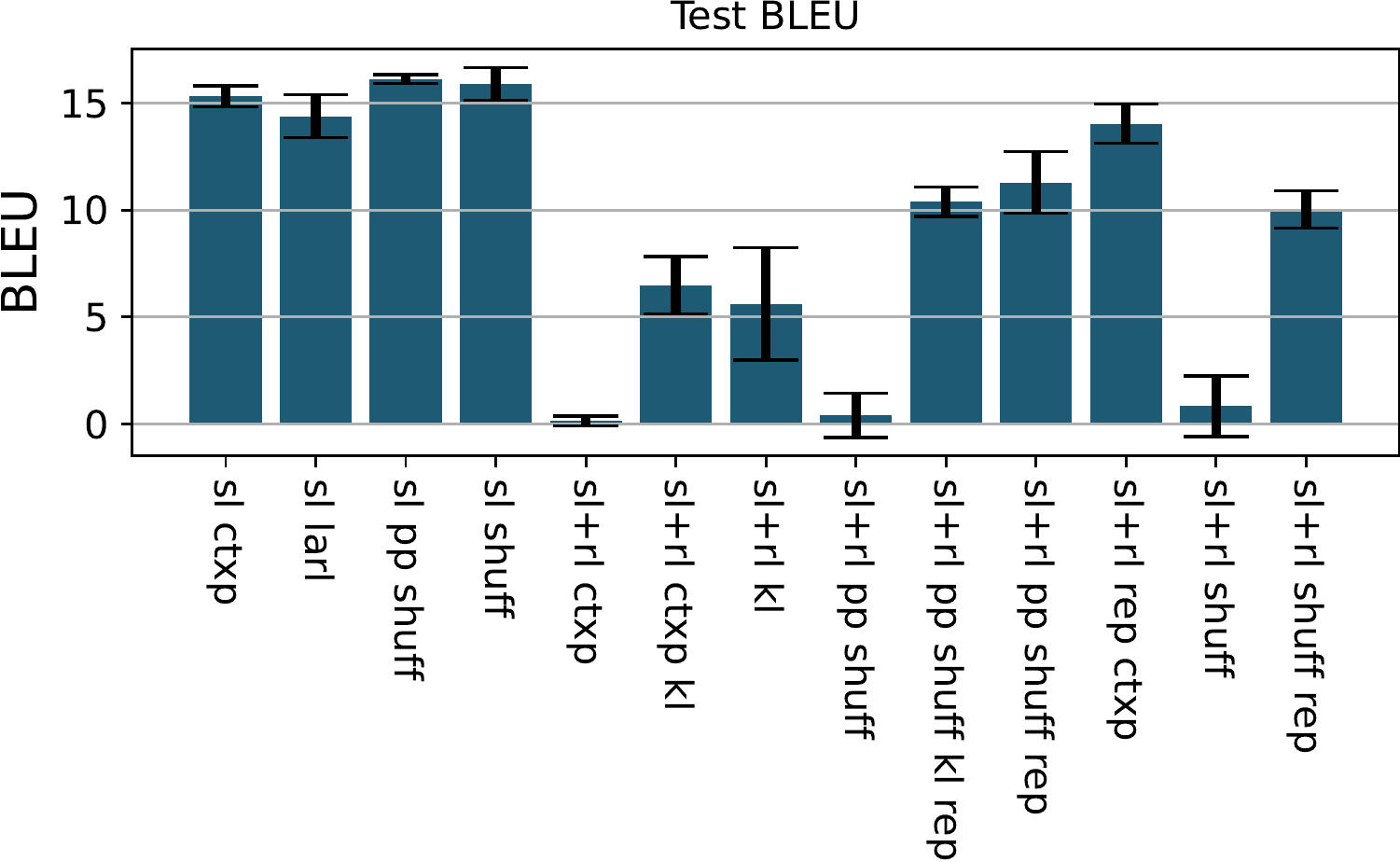}
    \end{subfigure}
    \caption{Ablation study. All \method{} configurations are averaged over three random seeds.
    \emph{sl} denotes supervised training, \emph{sl+rl} supervised training followed by reinforcement learning,
    \emph{ctxp} usage of contextual prior, but without any prior on it,
    \emph{kl} the addition of KL penalty term to identity covariance Gaussian,
    \emph{larl} the approach proposed by \citet{zhao:2019},
    \emph{shuff} prior and posterior shuffling in the supervised stage,
    \emph{pp} identity covariance Gaussian prior on the contextual prior,
    \emph{rep} usage of replay buffer
    }

    \label{fig:performance-hist}
  \end{figure}

\subsubsection*{Response Coherence}\label{sec:eval_bleu_deterioration}
We have observed that it is possible to maximize the dialogue success rate at the cost
of lower coherence in terms of BLEU scores.
In \figref{fig:bleu-deterioration} we show that we achieve state-of-the-art success rates,
whilst achieving a BLEU score of effectively 0.
This is an artifact of the success rate metric in that it disregards coherence of responses and only checks if the correct slot values have been addressed by the dialogue policy.

In \figref{fig:pareto-front} we see different runs of our method with different strength of regularization in terms of KL penalty and optimal replay fraction.
Depending on the strength
of regularization they form a Pareto front. This further shows the multi-objective
trade-off between success rates and BLEU scores.
\begin{figure}
  \centering
  \includegraphics[width=\linewidth]{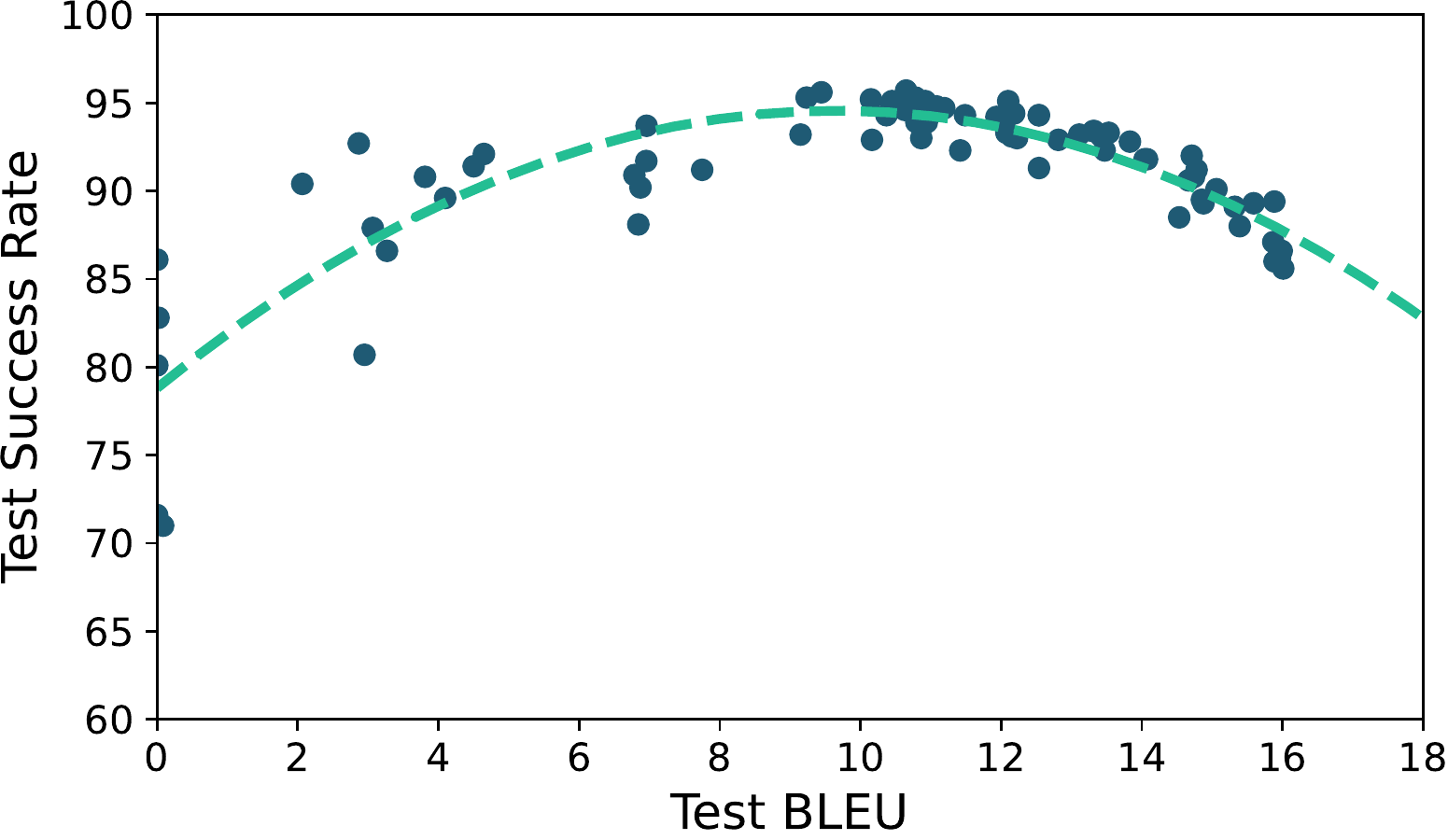}
  \caption{Each point depicts a training run and the line is a 2nd degree polynomial fit of the runs illustrating the Pareto front.}
  \label{fig:pareto-front}
\end{figure}
The BLEU collapse is to be expected when using latent-action reinforcement learning, since longer and more diverse answers
have positive impact on the dialogue success, which leads to the policy selecting degenerate responses.
Furthermore, if no additional regularization is used, the latent policy learns to sample outliers in terms of the prior over $\rz$.
We argue that it is not realistic to expect the policy resulting from the reinforcement learning stage to outperform the response coherence of the supervised learning stage in terms of BLEU (as long as the primary purpose of the RL stage is to improve dialogue success metrics).
Instead, we show that through regularization and optimal replay, the BLEU score can be kept from deterioration through the RL stage while we optimize dialogue level metrics.
An alternative approach to alleviating degenerate policies would be to simply make the BLEU score or other coherence metrics part of the reward function.
For example, the final comparison of models in \citep{budzianowski2018MultiWOZ} is done by comparing $\frac{success+inform}{2} + BLEU$.
However, we suggest that there are multiple problems with the approach of maximizing such a hybrid metric directly.
First of all, constructing reliable metrics is challenging \cite{jiang2021towards, mehri2020unsupervised} in itself.
More importantly, as soon as the coherence metric is part of the reward,
we encounter the problem of adequate scaling in comparison to the success rate.

\begin{figure}
  \centering
  \begin{tabular*}{\linewidth}{c}
    \includegraphics[width=1.0\linewidth]{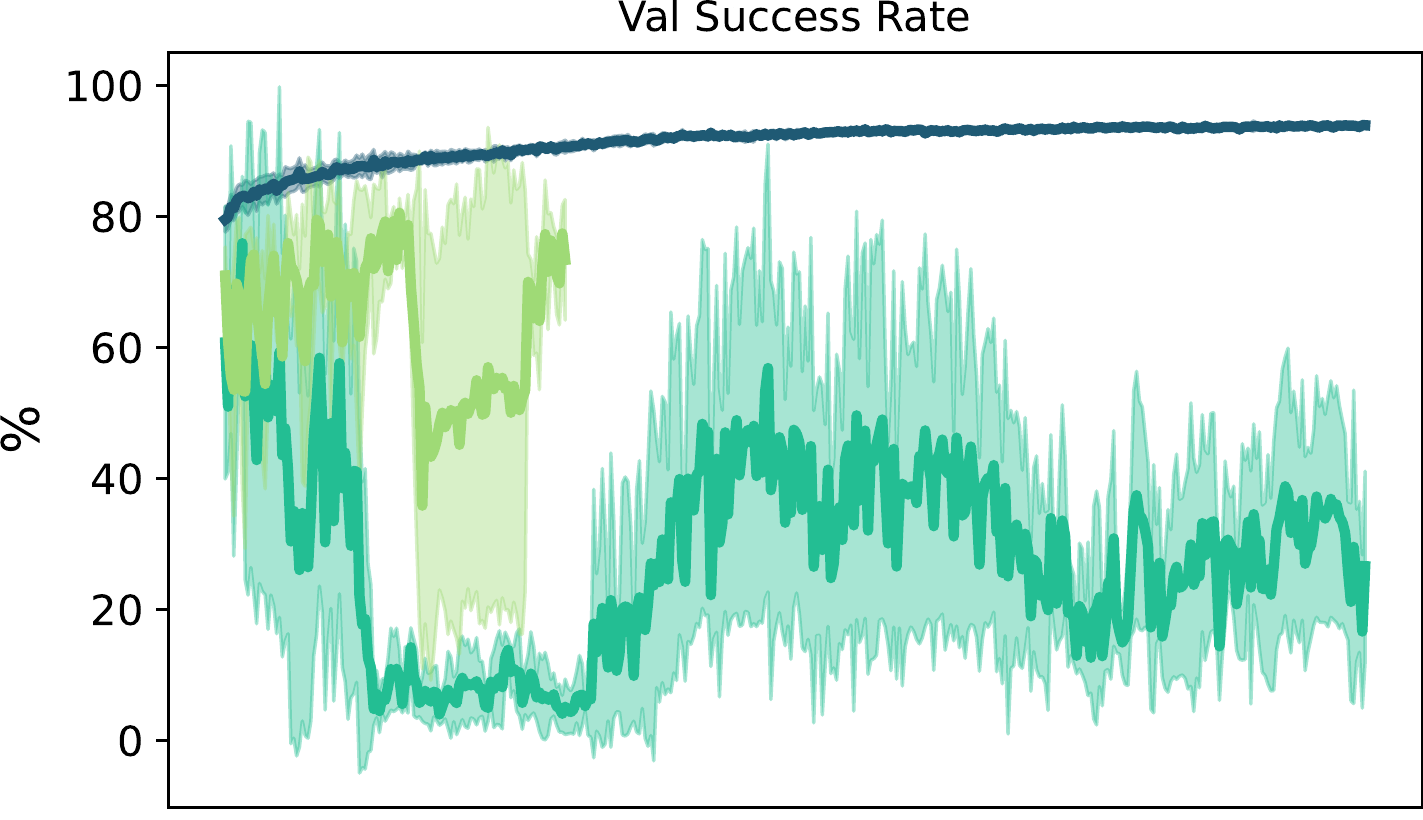} \\
    \includegraphics[width=1.0\linewidth]{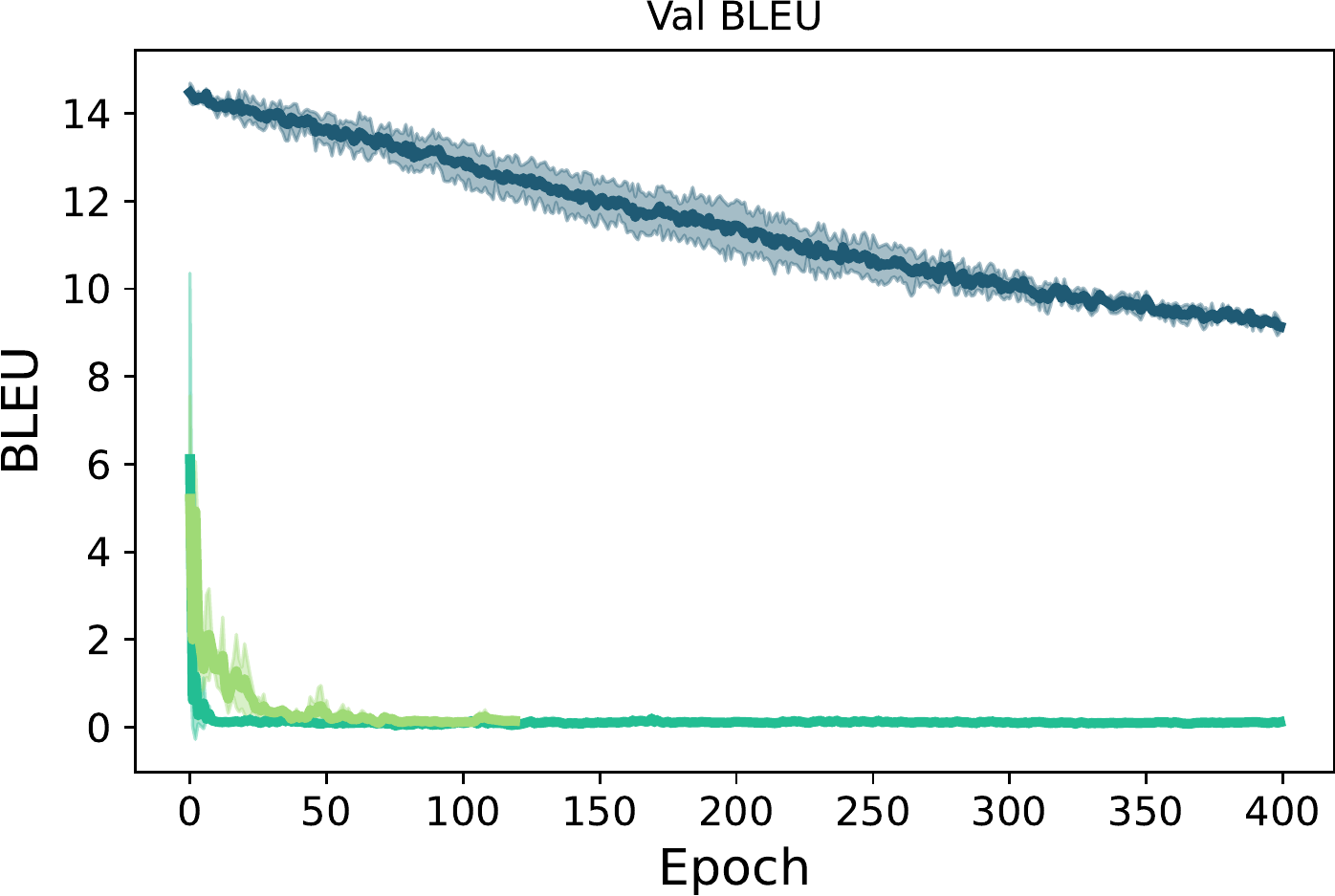}\\
    \includegraphics[width=1.0\linewidth]{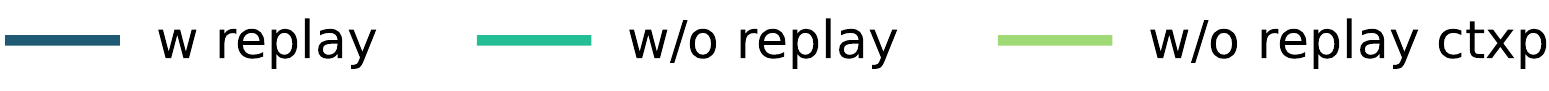}
  \end{tabular*}
  \caption{Influence of optimal replay on success rate and BLEU score over multiple epochs}
  \label{fig:bleu-deterioration}
\end{figure}

\paragraph{Impact of Regularization}
When introducing regularization in the form of optimal replay sampling and the KL penalty term, we are able to achieve state-of-the-art performance in terms of success rate without lowering the BLEU score significantly.
As depicted in \figref{fig:bleu-deterioration} our method using optimal replay buffer sampling is more stable during training in comparison to the naive application of the policy gradient.
Also, we observe that the BLEU deterioration is kept at bay and roughly deteriorates linearly over time.
By increasing  the replay fraction $\lambda$ too much we are over-constraining the latent dialogue policy to the initial experience, which leads to biased updates and hurts exploration.
Nevertheless, as shown in the ablation study (\figref{fig:performance-hist}) optimal replay and the KL penalty term have shown to be essential for preventing the BLEU score from deteriorating too rapidly.
In this work we aimed to retain a BLEU score that is competitive to that reported by LAVA~\citep{lubis:2020}, while improving overall dialogue success.
A sensitivity analysis over the weights of the penalty term and the
replay fraction $\lambda$ is described in Appendix \figref{app:fig:regularization-sensitivity}.

\begin{figure*}[h]
  \centering
  \begin{subfigure}{0.32\linewidth}
    \includegraphics[width=\linewidth]{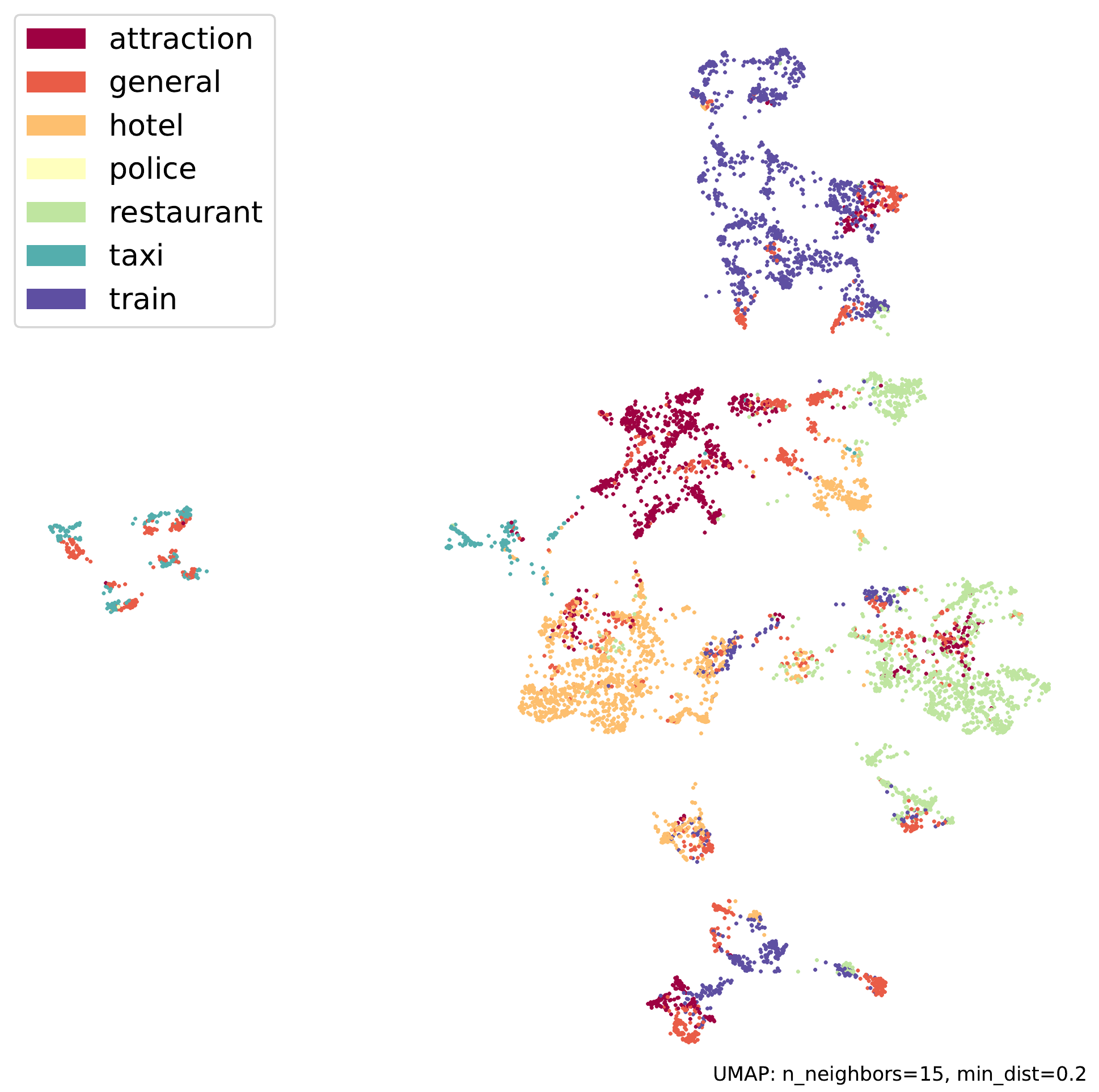}
    \caption{}
    \label{fig:umap-a}
  \end{subfigure}
  \begin{subfigure}{0.32\linewidth}
    \includegraphics[width=\linewidth]{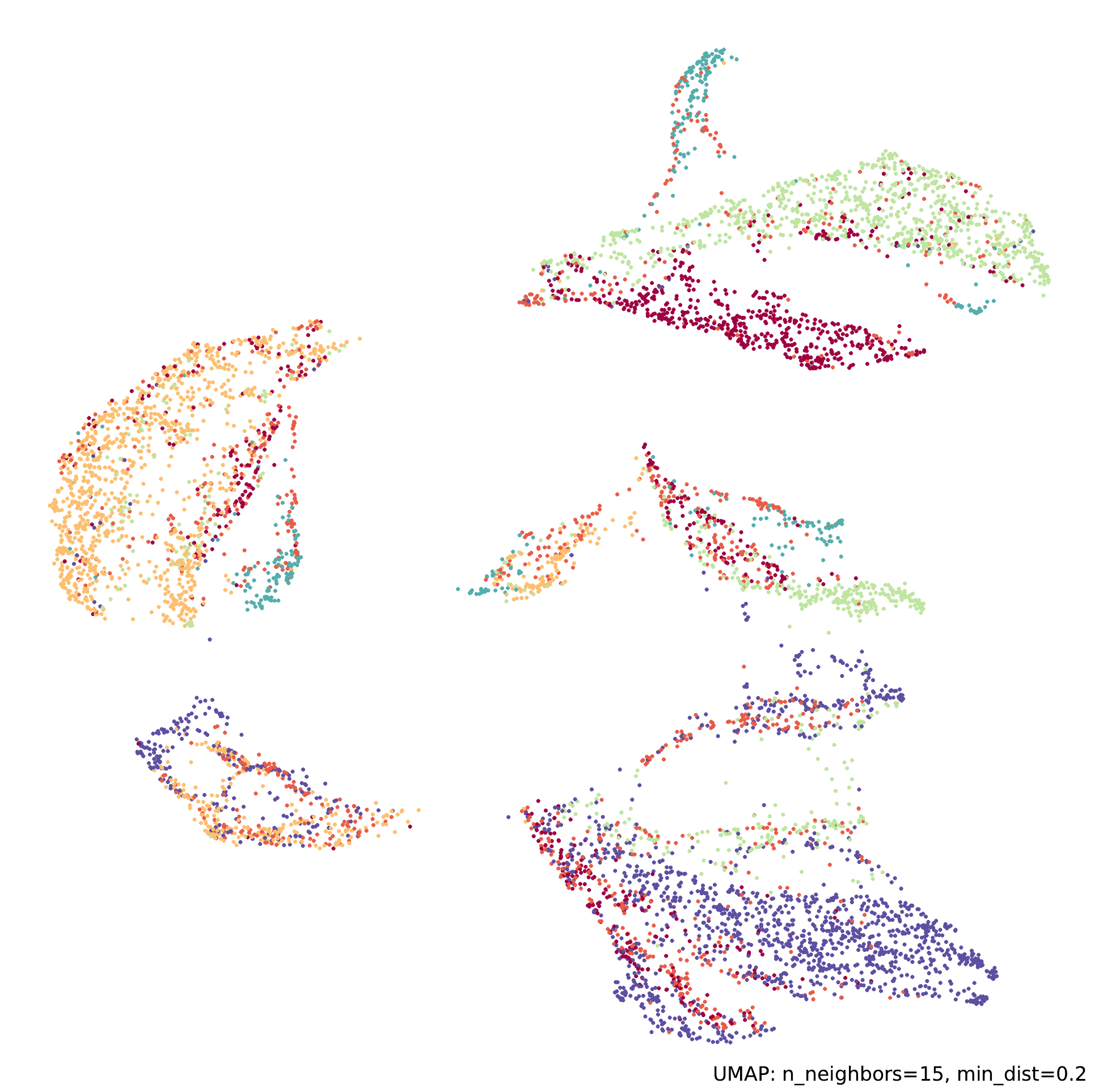}
    \caption{}
    \label{fig:umap-b}
  \end{subfigure}
  \begin{subfigure}{0.32\linewidth}
    \includegraphics[width=\linewidth]{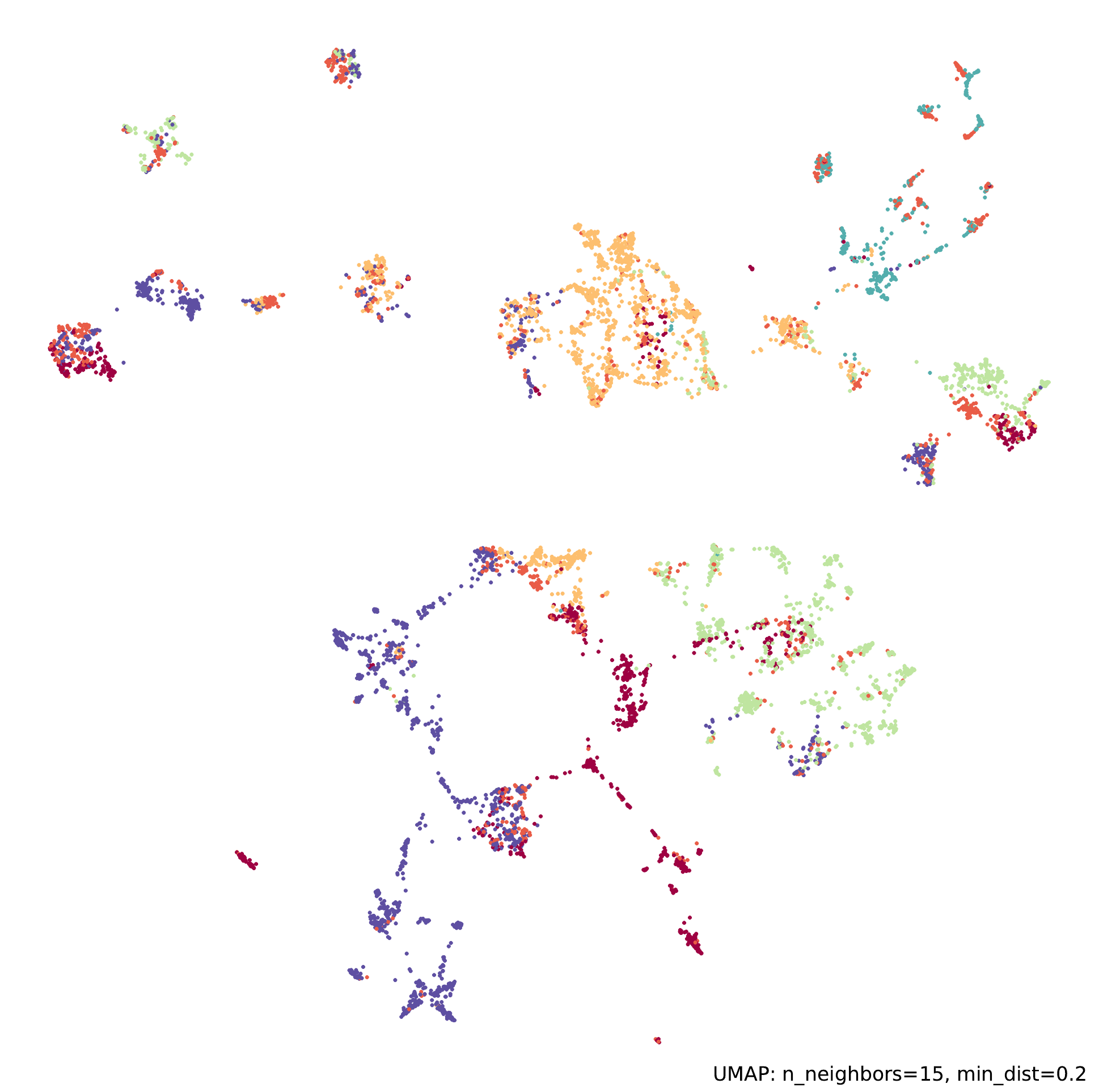}
    \caption{}
    \label{fig:umap-c}
  \end{subfigure}
  \begin{subfigure}{0.32\linewidth}
    \includegraphics[width=\linewidth]{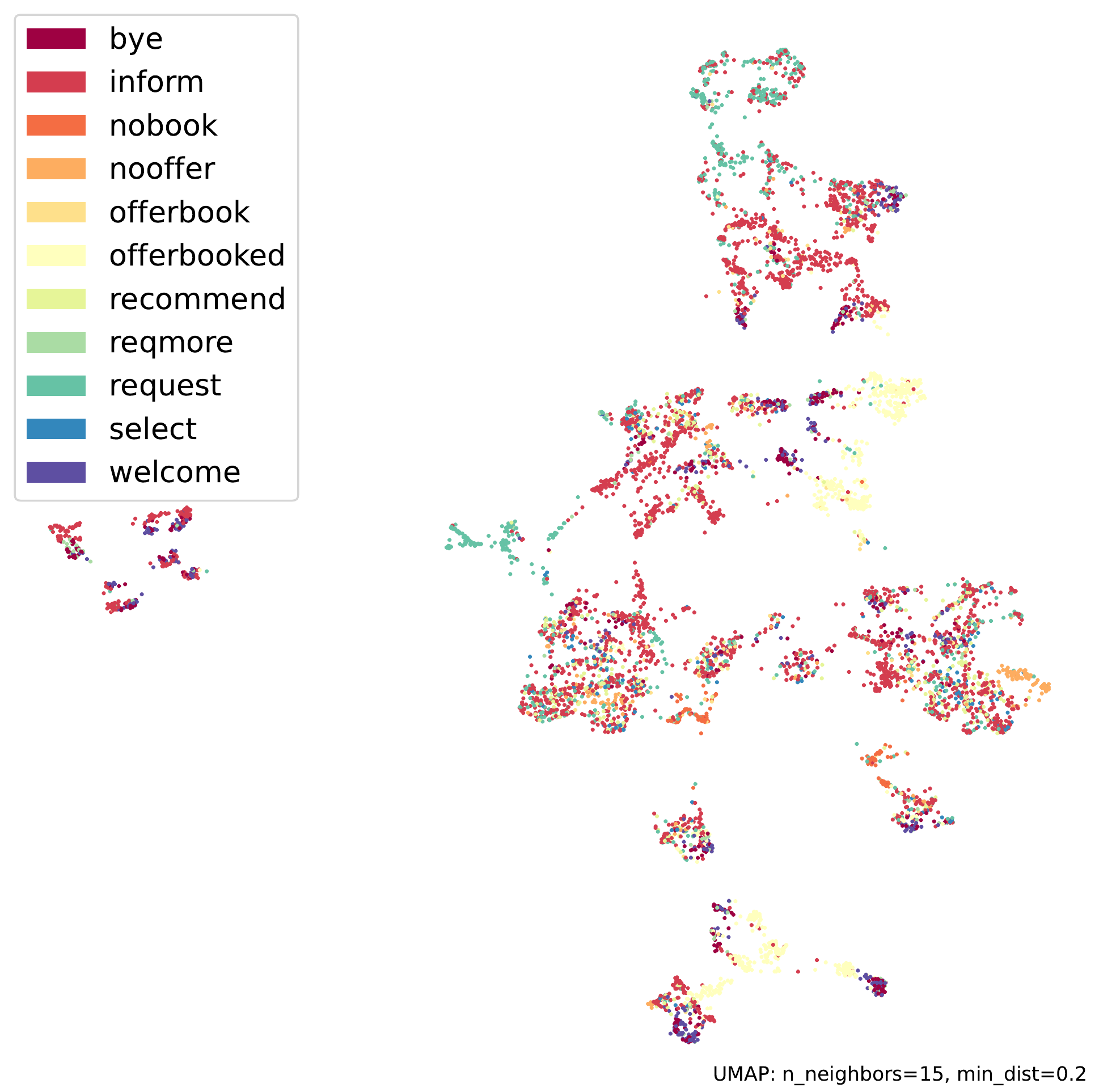}
    \caption{}
  \end{subfigure}
  \begin{subfigure}{0.32\linewidth}
    \includegraphics[width=\linewidth]{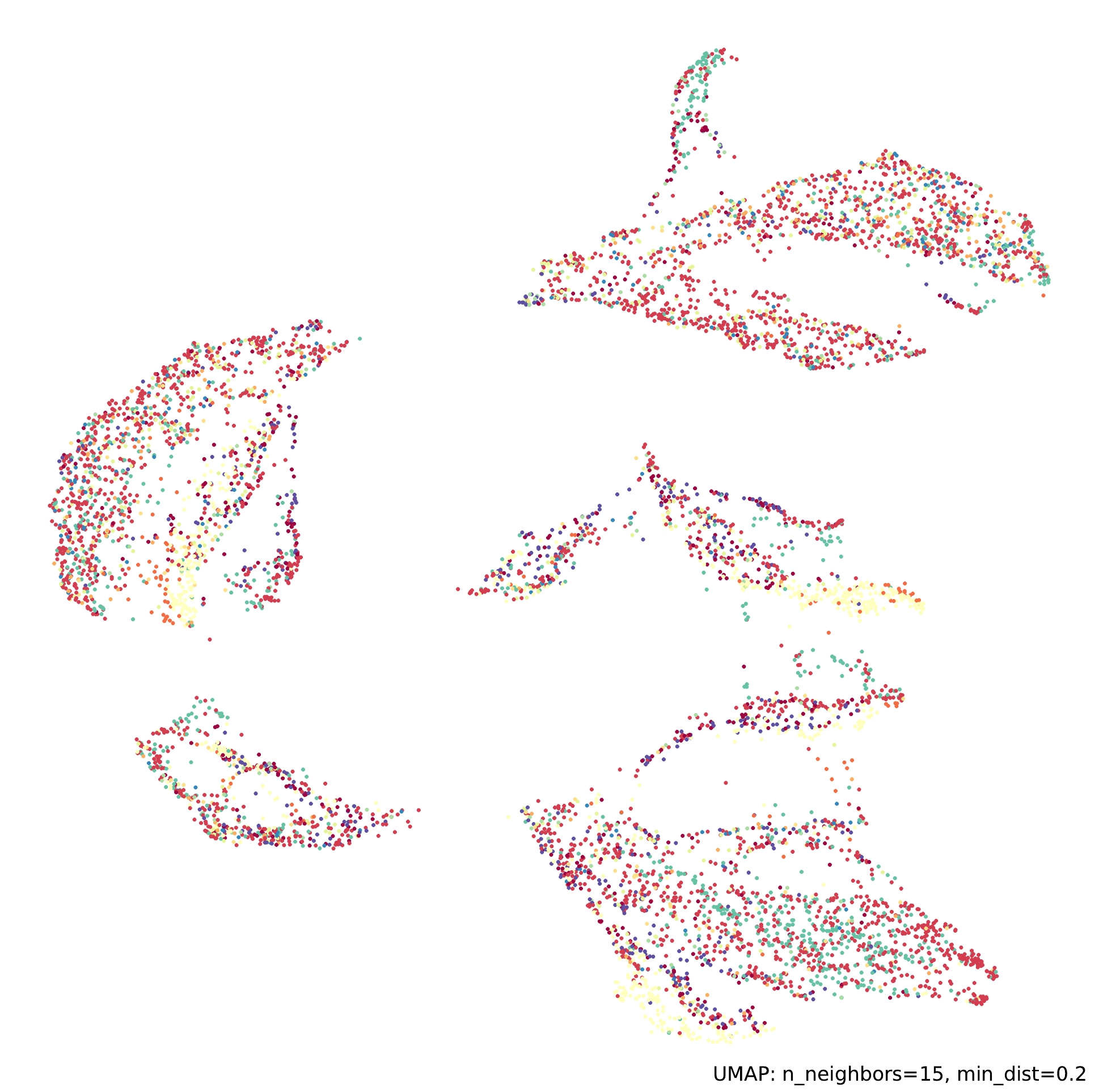}
    \caption{}
  \end{subfigure}
  \begin{subfigure}{0.32\linewidth}
    \includegraphics[width=\linewidth]{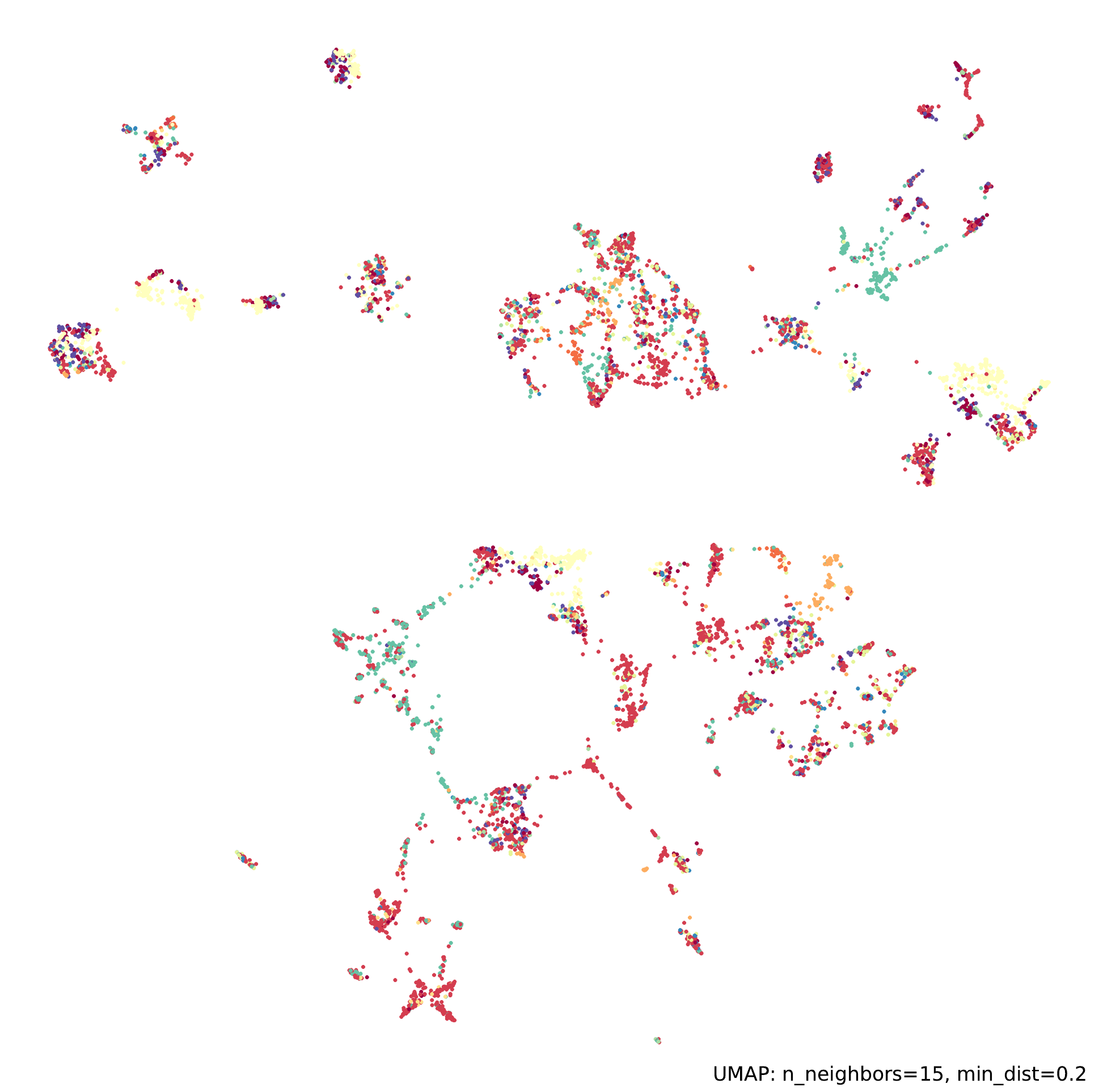}
    \caption{}
  \end{subfigure}
  \caption{UMAP embeddings of latent representations. Figures a-c show domain labeled embeddings for the SL, SL+RL and SL+RL with replay sampling, respectively.
  In the bottom row we find the same embeddings labeled by response type. Naively executing the RL stage of the training results in representations that are difficult to separate (b and e).
  When applying replay sampling we obtain higher specialization in the clusters (c and f).}
  \label{fig:umap-latent}
\end{figure*}

\subsection{Latent Space Analysis}\label{sec:latentanalysis}
\Figref{fig:umap-latent} depicts a UMAP\citep{mcinnes2018umap} projection of the learned latent samples $\rz$ for the Gaussian case.
We observe similar behavior to \citet{lubis:2020}.
In the supervised learning stage there is apparent clustering in terms of domain labels (\figref{fig:umap-a}).
The reinforcement learning stage with regularization leads to specialization of the clusters (\figref{fig:umap-c}).
In comparison, a good cluster separation is lost (\figref{fig:umap-b}) without applying regularization in the RL stage.
This can be explained by the fact that the $\rz$ samples are degenerate samples that lie in low-support regions of $\qp(\rz | \rc)$.

We have calculated the Cali{\'n}ski-Harabasz index \citep{calinski1974dendrite}, otherwise known as Variance Ratio Criterion, to evaluate the clusterings in the latent space with respect to domain and action type labelings taken from DAMD \citep{zhang2020task}.
In comparison to the results reported by \citet{lubis:2020}, our model is able to obtain high scores in the supervised stage of training already.
The scores drop slightly after RL fine-tuning, but generally remain at a much higher level than those reported for categorical latents.
It's important to note that the scores for LaRL$^*$ and LAVA$^*$ were computed with categorical latent variables, whereas our scores
are based on Gaussian latents, which are continuous and unbounded.

Consequently, the purpose of Table \ref{table:cluster} is to demonstrate the value in using continuous latent representations, rather than to make direct numeric comparisons.

\begin{table}
  \caption{Cali{\'n}ski-Harabasz scores (higher is better).}
  \centering
  \begin{tabular}{l | c | c | c | c | }
    \multirow{2}{*}{Model} & \multicolumn{2}{c}{SL} & \multicolumn{2}{c}{RL}  \\
        & Domain & Action   & Domain & Action \\
        \hline
        LaRL$^*$ & 93.19  &  23.20  &  121.15   &  17.5 \\
        \hline
        LAVA$^*$&  104.92 &   25.28  &  158.00 & 41.75  \\
        \hline
        \method{} &  345.48	&  80.76   &  329.39	& 78.95   \\
        \hline
      \end{tabular}
  \label{table:cluster}
\end{table}
\footnotetext[1]{Here we take the best scores of each method obtained by \citet{lubis:2020}}

\section{Related Work}\label{sec:relatedwork}
\noindent{}\textbf{Supervised Learning}
The majority of prior work applies some form of Supervised Learning.
For an extensive overview we refer to \citet{gao2018neural} and focus on
the state-of-the-art competitors HDSA~\cite{
chen2019semantically} as well as
UniConv~\cite{le2020uniconv}.
Both approaches demonstrate the benefits of
jointly training multiple dialogue tasks at once, such as predicting dialogue
acts and states.

\noindent{}\textbf{Variational Inference}
Several works employ variational inference for learning conditional response distributions.
VHRED~\cite{serban2017hierarchical} is a variational hierarchical RNN
architecture for modelling the dependencies between words and utterances.
Stochastic latent variables are introduced to generate the next utterance
in a conversation.
To avoid the problem of collapsing posteriors \citet{zhao2017learning} aim
to learn better response representation using an auxiliary task introducing
a bag of words prediction loss.
\citet{shen2018improving} improves the stability of VHRED by splitting the training
process into two parts, where first text is auto-encoded into continuous
embeddings, which serve as starting point for learning general latent
representations by reconstructing the encoded embedding.
To address the degeneration problem, Variational Hierarchical Conversation
RNNs (VHCR)~\citep{park2018hierarchical} exploits an utterance drop
regularization.
DialogWAE~\citep{gu2018dialogwae} represents the prior distribution as
Gaussian mixture and adapts WGAN~\citep{arjovsky2017wasserstein} for
training.

\noindent{}\textbf{Combined Supervised- \& Reinforcement Learning}
\citet{henderson2008hybrid} was the first work to combine Reinforcement and
Supervised Learning introducing a value function, which relies on SL for
predicting the expected future reward of states not covered by the data directly.
\citet{williams2017hybrid} prevent high variance by performing a SL gradient
update, if the RL
policy output deviates from the training data.
\citet{fatemi2016policy, su2017sample} apply two-stage approaches, where they
pretrain a dialogue policy supervised, which is then further
optimized by RL.
HappyBot~\citep{shin2019happybot} relies on a weighted combination of a
maximum likelihood and a REINFORCE objective.
\citet{saleh2020hierarchical} uses REINFORCE-based policy gradients to update
the prior probability distribution of the latent variational model trained using
supervised learning.
Structured Fusion Networks~(SFNs)~\citep{mehri2019structured} apply RL
to fuse dialogue modules, where each module serves a different purpose
(NLU, NLG, etc.) and is pretrained in individual supervised stages.
Apart from task-oriented dialogues a combination of SL and RL has also been
applied for generating responses for open-domain dialogues \citep{xu2018towards}.
The aforementioned approaches are based on word-level RL, suffering from huge
action-spaces covering the entire input vocabulary.
Due to this, ensuring coherent responses is challenging \citep{lewis2017deal, kottur2017natural}, especially in multi-turn dialogs spanning hundreds of words.
LaRL~\citep{zhao:2019} and LAVA~\citep{lubis:2020} address this problem by learning a low-dimensional latent representation
with  amortized variational inference followed by a RL fine-tuning stage where the decoder parameters are frozen and the encoder is
fine-tuned.

\section{Conclusion}
In this paper, we showed that with appropriate modifications to the reinforcement learning procedure and the latent space structure, it is possible to obtain state-of-the-art results in task-oriented dialogue with simple Gaussian latent distributions.
The problem of coherence deterioration when optimizing for success rate points to the fact
that better metrics are needed for developing efficient dialogue policies - we were able to alleviate this with optimal replay and KL regularization.

Although \method{} has shown promising results in the policy learning setting with access to ground-truth dialogue state, it would be of interest to see
how \method{} compares in the end-to-end learning setting.
Furthermore, the deterioration of the BLEU score is tightly connected to the capacity of the model to achieve good likelihood fits in the supervised learning stage.
Our experiments also hint that this variational formulation can benefit from higher-capacity models.
Along these lines, increasing capacity and making use of the transformer architecture is a promising future research direction.
\bibliography{bib}

\include{appendix}

\end{document}

%% file: math_commands.tex
\usepackage{amsmath,amsfonts,bm}

\def\figref#1{figure~\ref{#1}}
\def\Figref#1{Figure~\ref{#1}}

\def\secref#1{section~\ref{#1}}
\def\Secref#1{Section~\ref{#1}}

\def\eqref#1{equation~\ref{#1}}

\def\1{\bm{1}}

\def\ra{{\textnormal{a}}}

\def\rc{{\textnormal{c}}}

\def\rs{{\textnormal{s}}}

\def\rx{{\textnormal{x}}}

\def\rz{{\textnormal{z}}}

\DeclareMathAlphabet{\mathsfit}{\encodingdefault}{\sfdefault}{m}{sl}
\SetMathAlphabet{\mathsfit}{bold}{\encodingdefault}{\sfdefault}{bx}{n}

\def\gL{{\mathcal{L}}}

\newcommand{\E}{\mathbb{E}}

\newcommand{\KL}{D_{\mathrm{KL}}}

\DeclareMathOperator*{\argmin}{arg\,min}

%% file: our_math.tex
\def\mdp{{\mathcal{M}\mathcal{D}\mathcal{P}}}

\usepackage{amsthm}

\newtheorem*{remark}{Remark}

\newtheorem{lemma}{Lemma}
\newtheorem{proposition}{Proposition}

%% file: appendix.tex
\appendix

\def\comments{}

\section{Revisiting the Full ELBO}
\label{sec:app:math}

Further we use shorthand $q$ for $q(\rz | \rx, \rc)$
Firstly we define our optimization problem:
\begin{align*}
\min_q  \KL[ q || p(\rz | \rx, \rc)]
\end{align*}

\begin{figure}[H]
    \centering
    \begin{tikzpicture}
        \Vertex[x=0, label=$\rc$, color=white]{c}
        \Vertex[x=1.5, label=$\rz$, color=red]{z}
        \Vertex[x=3, label=$\rx$, color=green]{x}
        \Edge[Direct=true](c)(z)
        \Edge[Direct=true](z)(x)
    \end{tikzpicture}
    \caption{Probabilistic model of dependencies between $\rc$, $\rz$ and $\rx$}
    \label{fig:prob-model}
\end{figure}
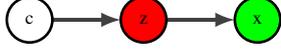

\begin{lemma}
    Given the probabilistic model diagram \ref{fig:prob-model}:
    \begin{align*}
        &\KL[ q || p(\rz | \rx, \rc)] \le  \E_q[\log q] -\\&\E_q[\log p(\rx | \rz)] - \E_q[\log p(\rz | \rc)]
    \end{align*}
\end{lemma}

\begin{proof}

We expand the KL divergence of \eqref{eq:min-kl}.
\begin{align}\label{eq:min-kl}
\KL[ q || p(\rz | \rx, \rc)] = \E_q [\log q - \log p(\rz | \rx, \rc)]
\end{align}

\newcommand{\ind}{\perp\!\!\!\!\perp}
We first expand the $p(\rz | \rx, \rc)$ by Bayes rule and make use of independence between $x$ and $c$:
\begin{align*}
    & p(\rz | \rx, \rc)  = \frac{p(\rx, \rc | \rz) p(\rz)}{p(\rx,\rc)}\\
    &= \frac{p(\rx | \rz) p(\rc | \rz) p(\rz)}{p(\rx,\rc)}, \;  \rx \ind \rc |\, \rz\\
    &= \frac{p(\rx | \rz)p(\rz | \rc) p(\rc)}{p(\rx,\rc)}
\end{align*}
Putting back these terms into \eqref{eq:min-kl}:
\begin{align*}
    &\KL[ q || p(\rz | \rx, \rc)] = \\
    &\E_q[\log q] -\E_q[\log p(\rx | \rz)] - \E_q[\log p(\rz | \rc)] \\
    &- \E_q[\log p(\rc)]  + \E[\log p(\rx,\rc)] \\
    & = \E_q[\log q] -\E_q[\log p(\rx | \rz)] - \E_q[\log p(\rz | \rc)] \\
    &- \log p(\rc) + \log p(\rx,\rc) \\
    & \le \E_q[\log q] -\E_q[\log p(\rx | \rz)] \\
    & - \E_q[\log p(\rz | \rc)] \, p(\rc) \ge p(\rx,\rc)
\end{align*}

\end{proof}

In practice, we don't have access to $p(\rz|\rc)$ and hence we estimate it with $q^p_\theta(\rz | \rc)$, which
results in an $\argmin$ operation within the expectation.
\def\qp{{q^p_\zeta}}
\def\q{{q_\theta}}
\def\p{{p_\phi}}

\begin{align*}
    &\E_q[\log q] -\E_q[\log p(\rx | \rz)] - \E_q[\log p(\rz | \rc)] = \\
    & \E_q[\log q] -\E_q[\log \argmin_{q^p}\KL[q^p || p(\rz | \rc)]] \\&- \E_q[\log p(\rx | \rz)]
\end{align*}
The new KL term within the expectation can be expanded to obtain the traditional evidence lower-bound, ie. bound on the objective:
\begin{align*}
    &\KL[q^p || p(\rz | \rc)] \le \KL[\qp || p(\rz)] \\&+ \E_\qp[p(\rc | \rz)]
\end{align*}

In the end, this is our objective function, but we want to get rid of $\p( \rc | \rz)$.
\begin{align*}
    \gL(\phi, \theta, \zeta) &= \KL[ q_\theta || \qp] + \KL[\qp || p(\rz)] - \\
    &\E_\q[\log \p(\rx | \rz )] - \E_\qp[\log p(\rc | \rz)]
\end{align*}

The current formulation would require us to approximate $p(\rc | \rz)$.
\begin{remark}
    The KL term $ \KL[ p(\rz | \rc) || p(\rz | \rx,\rc)] $ is zero iff  $\rx$ doesn't provide any additional information
    about $\rz$.
\end{remark}

We further simplify the problem by replacing $p(\rc | \rz)$ with $\p(\rx | \rz)$.
By applying Bayes rule, we arrive to the following:
\begin{align*}
    \gL(\phi, \theta, \zeta) &= \KL[ q_\theta || \qp] + \KL[\qp || p(\rz)] - \\
    &\E_\q[\log \p(\rx | \rz )] - \E_\qp[\log \p(\rx | \rz)] \\
    & - \E_\qp[\log \frac{p(\rc| \rz)}{p_\theta(\rx | \rz)}]
\end{align*}

\begin{lemma}\label{lemma-approx}
    $p(z|x,c) \approx E_{p(c)}[p(z|x,c)]$ implies that $p(c | z) \approx \frac{p(c)}{p(x)} p(x|z) $
\end{lemma}

\begin{proof}
    \begin{align*}
        & p(z|x,c) \approx E_{p(c)}[p(z|x,c)] \\
        & p(c|z) = p(z|c)p(c)/ p(z) \\
        & p(x|z) = p(z|x)p(x)/ p(z) \\
        & \frac{p(c|z)}{p(x|z)} \approx \frac{p(c) p(z|c,x)}{p(x)\E_{p(c)}[p(z|x,c)]} \\
        & \frac{p(c|z)}{p(x|z)} \approx \frac{p(c)}{p(x)}
    \end{align*}

\end{proof}

\begin{proposition}
    From the above it follows that for the optimization problem $\min_q  \KL[ q || p(\rz | \rx, \rc)]$ it suffices to
    optimize the following objective:
    \begin{align*}
        \gL(\phi, \theta, \zeta) &= \KL[ q_\theta || \qp] + \KL[\qp || p(\rz)] - \\
        &\E_\q[\log \p(\rx | \rz )] - \E_\qp[\log \p(\rx | \rz)] + C
    \end{align*}
\end{proposition}


How accurate the approximation is depends highly on on the relationship between $p(\rz | \rx)$ and $p(\rz | \rc)$.

\begin{lemma}
    In the case where $p(\rz | \rx) = p(\rz | \rc)$, the approximation from \ref{lemma-approx} holds with equality.
\end{lemma}
\begin{proof}
    This can be easily seen by replacing the $p(\rz | \rx)$ and $p(\rz | \rc)$ in the proof of Lemma \ref{lemma-approx},
    we arrive to the same result without making the approximation step.
\end{proof}
Intuitively,  $p(\rz | \rx)$ and $p(\rz | \rc)$ in the ideal case would be equal, i.e. knowing either $\rx$ or $\rc$ would be enough to obtain all the information
about $\rz$, which allows for good information compression.

In our approach, we optimize $\E_\qp[\p(\rx | \rz)]$ and $\E_\q[\p(\rx | \rz )]$ by interchangebly sampling from $\p$ and $\qp$ based on a
fair coin flip, ie Bernoulli distribution and a assume an identity covariance Gaussian prior $p(\rz)$.

\section{Evaluation Histograms}
\label{app:histograms}

In \figref{app:fig:eval-hist} we show histograms of different evaluation metrics for the MultiWOZ 2.1 dataset and different approaches.
We ran 3 seeds for each of the methods.
Most notably, we notice that we are able to obtain the highest BLEU score in the supervised learning stage by utilizing prior and posterior shuffling, denoted as \emph{sl shuff}, with
the result slightly increasing when using an additonal identity covariance Gaussian prior on the prior (\emph{sl pp shuff}).

\begin{figure*}
    \newlength{\asubfig}
    \setlength{\asubfig}{{0.3\linewidth}}
    \begin{subfigure}{\asubfig}
        \includegraphics[width=\linewidth]{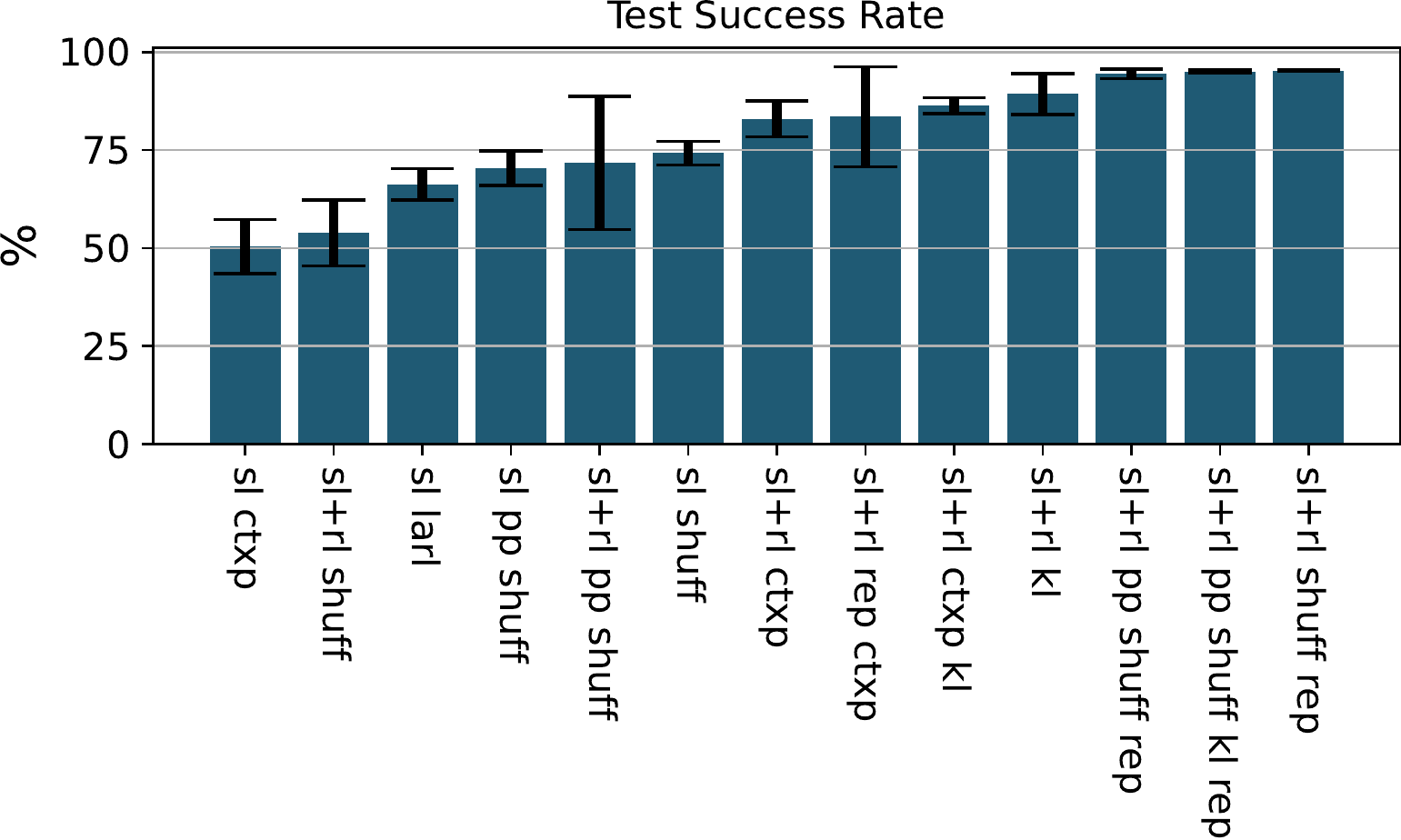}
        \caption{}
    \end{subfigure}
    \begin{subfigure}{\asubfig}
        \includegraphics[width=\linewidth]{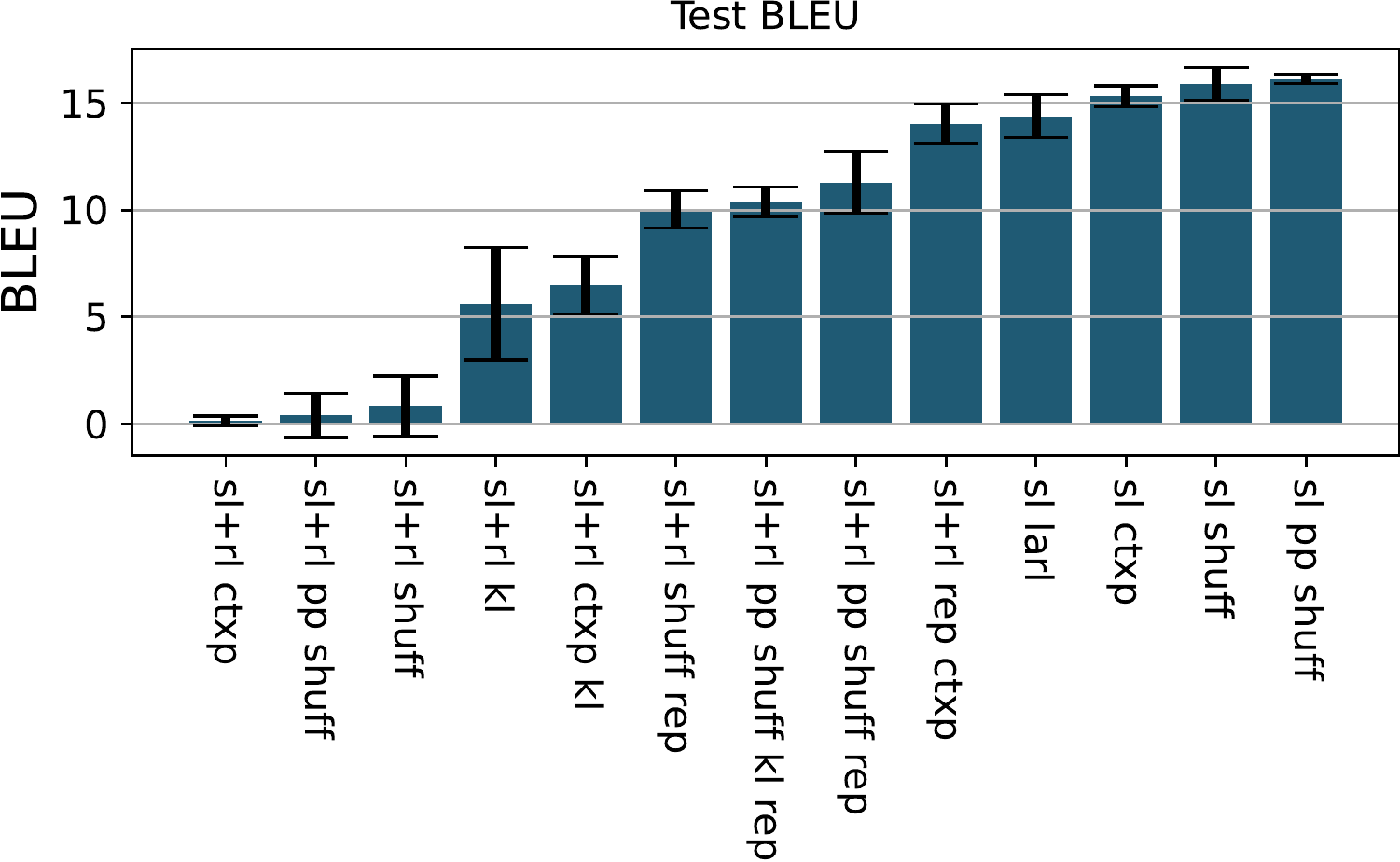}
        \caption{}
    \end{subfigure}
    \begin{subfigure}{\asubfig}
        \includegraphics[width=\linewidth]{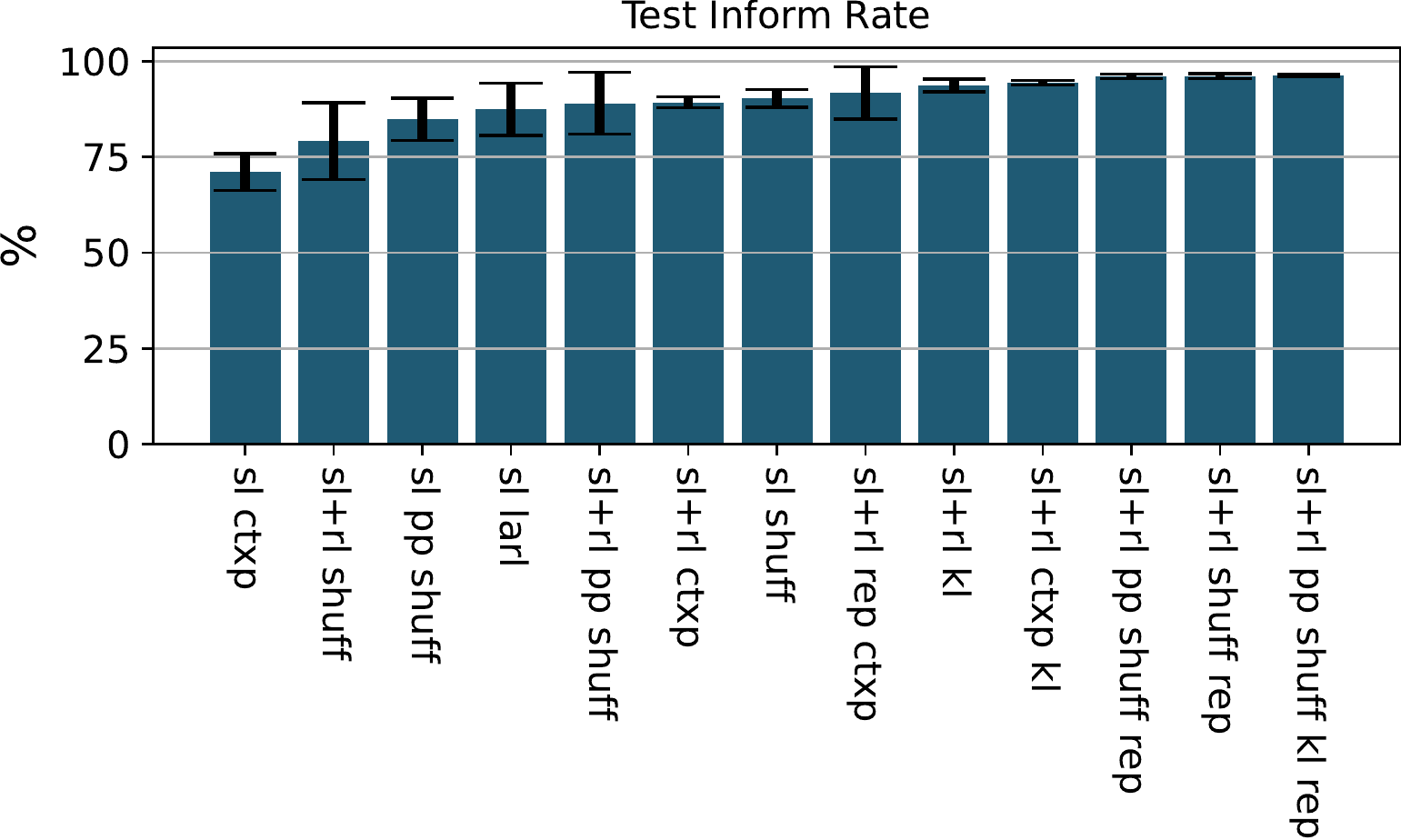}
        \caption{}
    \end{subfigure}

    \begin{subfigure}{\asubfig}
        \includegraphics[width=\linewidth]{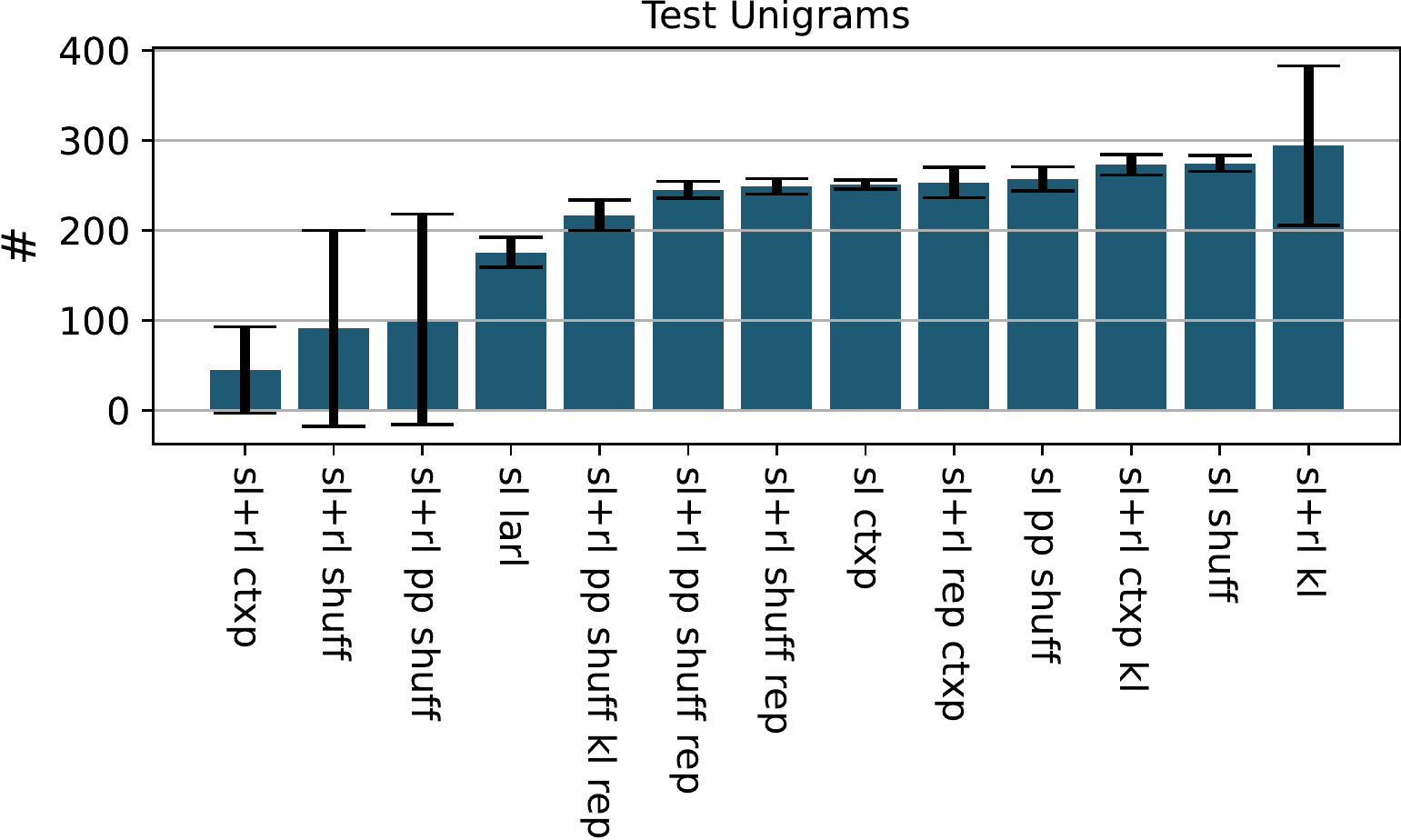}
        \caption{}

    \end{subfigure}
    \begin{subfigure}{\asubfig}
        \includegraphics[width=\linewidth]{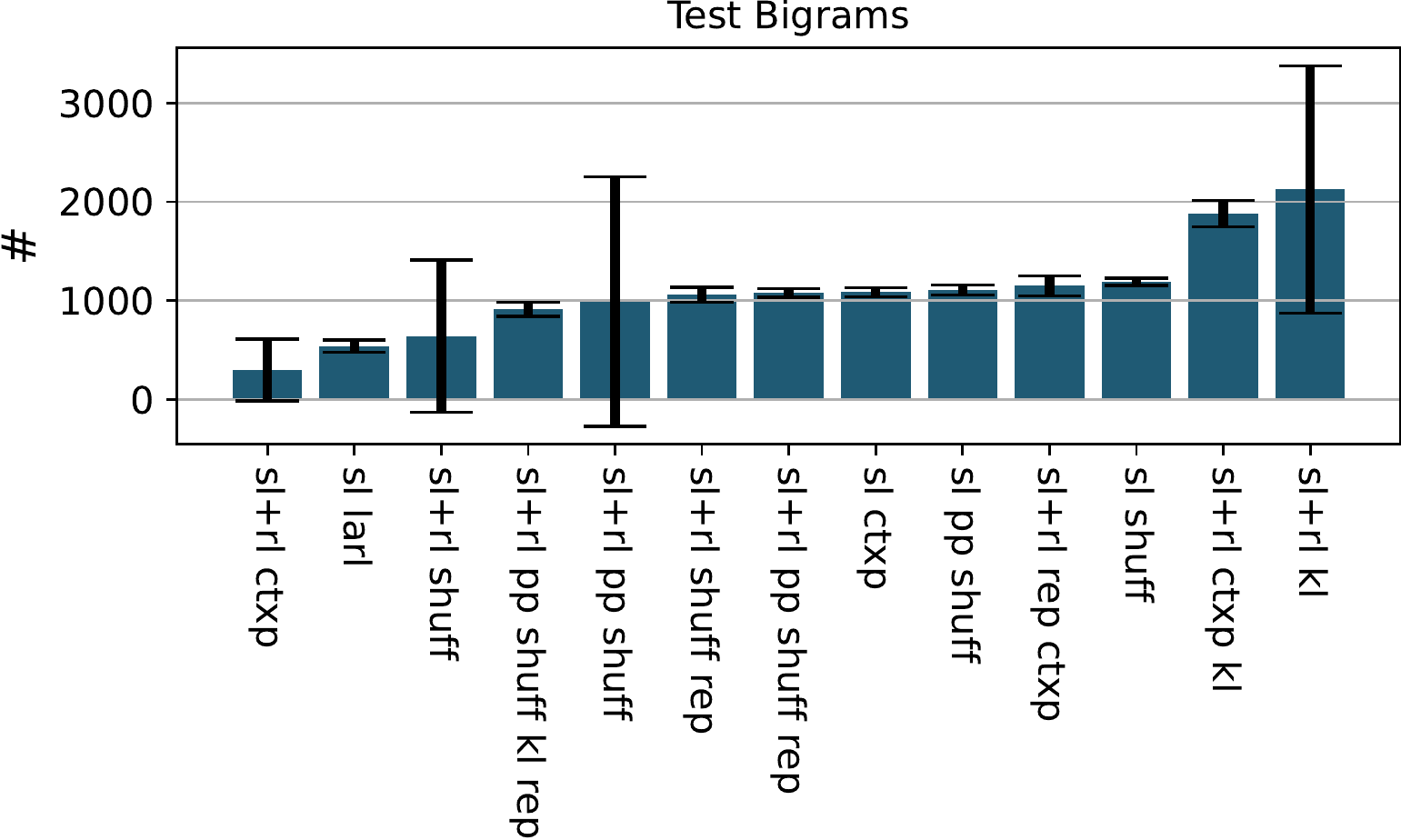}
        \caption{}

    \end{subfigure}
    \begin{subfigure}{\asubfig}
        \includegraphics[width=\linewidth]{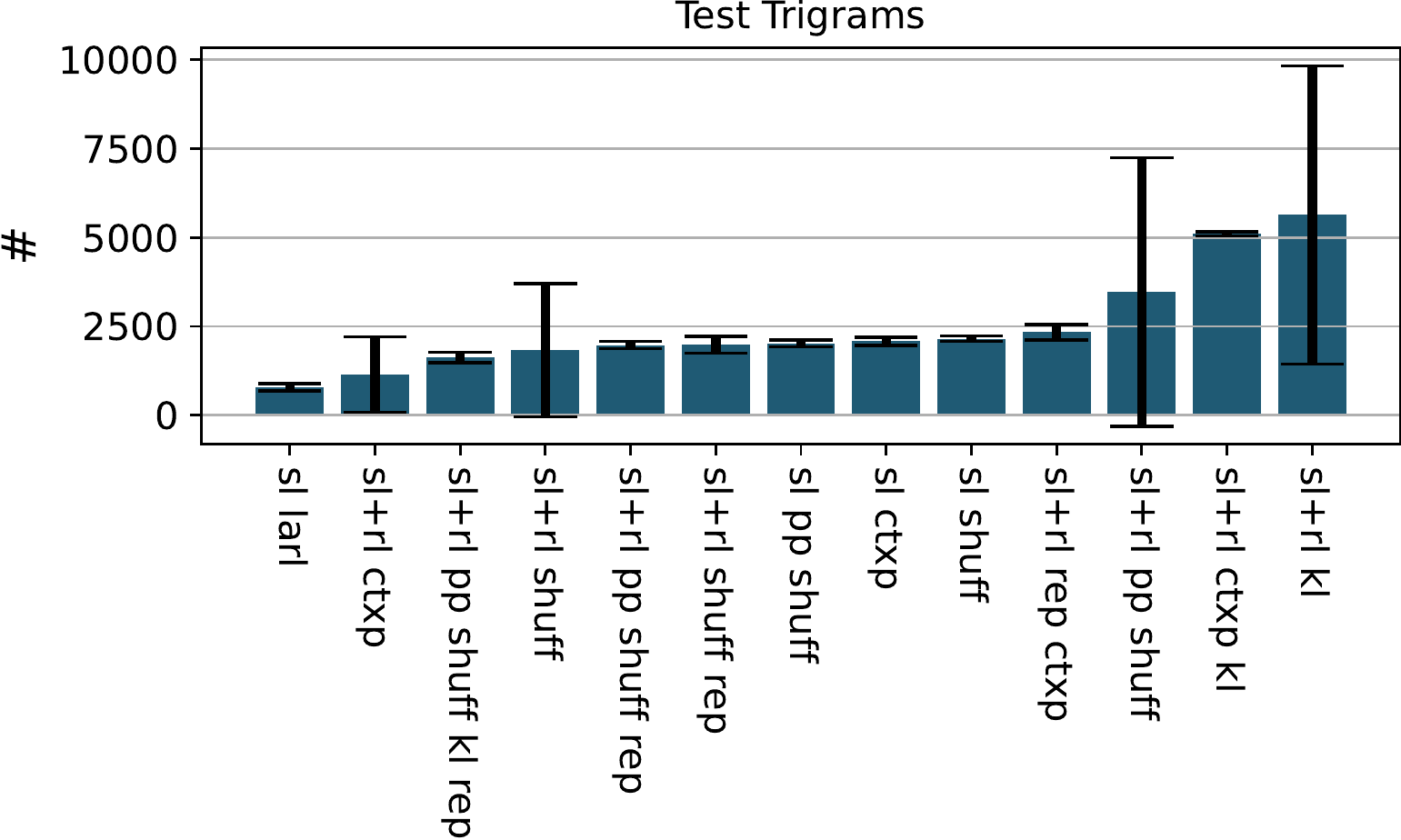}
        \caption{}

    \end{subfigure}

    \caption{
        Here we show results for different metrics on MultiWOZ 2.1 for different approaches denoted on the x-axis.
        \emph{sl} denotes supervised trainig, \emph{sl+rl} denotes supervised training followed by reinforcement learning, \emph{ctxp}
        usage of contextual priror, but without any prior on it, \emph{kl} the addition of KL penalty to identity covariance Gaussian,
        \emph{larl} the approach proposed by \citet{zhao:2019}, \emph{shuff} prior and posterior shuffling in the supervised stage, \emph{pp}
        identity covariance Gaussian prior on the contextual prior, \emph{rep} usage of replay buffer for the reinforcement learning fine-tuning stage of training.
    }
    \label{app:fig:eval-hist}
\end{figure*}

\begin{table*}[!th]
  \centering
  \begin{tabular}{ | l | p{12cm} |}
    \hline
    USER & hi . i was wondering if there are any eastern [value\_food] type restaurant -s in the [value\_area] part of the city?\\
    \hline
    SYSTEM & [restaurant\_name] is a [value\_food] restaurant in the [value\_area] of town . it is [value\_pricerange] -ly priced and located at [restaurant\_address]. would you like to book a table? \\
    \hline
    USER & i am looking for an [value\_pricerange] eastern [value\_food] place in the [value\_area] \\
    \hline
    SYSTEM & [restaurant\_name] is a [value\_pricerange] [value\_food] restaurant in the [value\_area] of town . would you like me to book a table for you? \\
    \hline
    USER & how about a [value\_food] restaurant? \\
    \hline
    SYSTEM & i have [value\_count] [value\_pricerange] [value\_food] restaurant -s in the [value\_area] [restaurant\_name], [restaurant\_name], [restaurant\_name], [restaurant\_name], [restaurant\_name], [restaurant\_name], [restaurant\_name], [restaurant\_name], [restaurant\_name], [restaurant\_name], [restaurant\_name], [restaurant\_name], [restaurant\_name], [restaurant\_name], [restaurant\_name], [restaurant\_name], [restaurant\_name], [restaurant\_name], [restaurant\_name], [restaurant\_name] \\
    \hline
    USER & i would like a reservation for [value\_count] to the [restaurant\_name] restaurant.\\
    \hline
    SYSTEM & i have booked you a table for [value\_count] at [restaurant\_name] on [value\_day] at [value\_time]. your reference number is [restaurant\_reference].\\
    \hline
    USER & i would like to make a reservation for [value\_day] at [value\_time] . and there has been a change in plans , i will be dining alone.\\
    \hline
    SYSTEM & your booking was successful . the table will be reserved for [value\_count] minutes . reference number is [restaurant\_reference]. \\
    \hline
    USER & can i get the reference number for that reservation ? \\
    \hline
    SYSTEM & your booking was successful . the table will be reserved for [value\_count] minutes . reference number is [restaurant\_reference]. \\
    \hline
    USER & excellent . that is everything! \\
    \hline
    SYSTEM & you are welcome . have a great day!\\
    \hline
  \end{tabular}
  \caption{Example of resulting delexicalized dialogue outputted by learned dialogue policy learned with optimal replay.}
  \label{app:tb:dialogue}
\end{table*}

\section{LAVA Gauss Results}
\label{app:sec:lava-gauss}

To obtain scores with the newest evaluation method for MultiWOZ, we ran the code available at \url{https://gitlab.cs.uni-duesseldorf.de/general/dsml/lava-public}.
We didn't make any modification to the hyperparameters and ran the reinforcement learning stage for 10 epochs.
Notably, the LAVA method with continuous latent exibits high variance in the success scores, with a standard deviation of $3.32$ and mean success rate $73.33$ across 3 runs,
it is able to keep the BLEU score high since it's under-optimizing the success metric and is trained for few epochs which prevents divergence.
In fact, \method{} can obtain the same BLEU score by sacrificing a bit on the success rate side, see \figref{fig:pareto-front}.

\section{Regularization Sensitivity}

In \figref{app:fig:regularization-sensitivity} we do a regularization sensitivity analysis for different scales of
KL regularization and fraction of optimal replay in the reinforcement learning training stage.
We do the same for the categorical latent variable case in \figref{app:fig:categorical-gs}.

\begin{figure*}
    \newlength{\bsubfig}
    \setlength{\bsubfig}{{0.3\linewidth}}
    \begin{subfigure}{\bsubfig}
        \includegraphics[width=\linewidth]{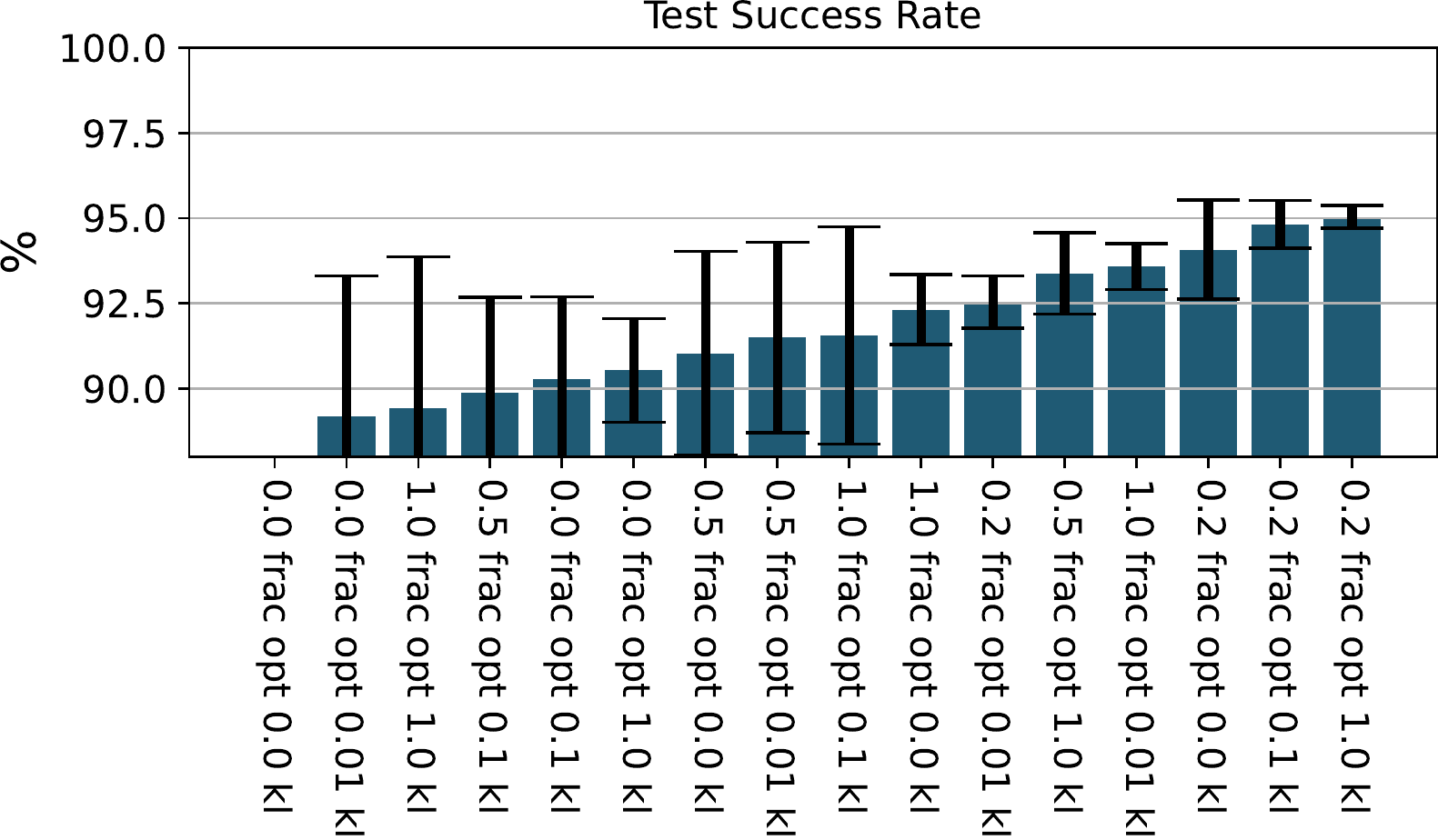}
        \caption{}
    \end{subfigure}
    \begin{subfigure}{\bsubfig}
        \includegraphics[width=\linewidth]{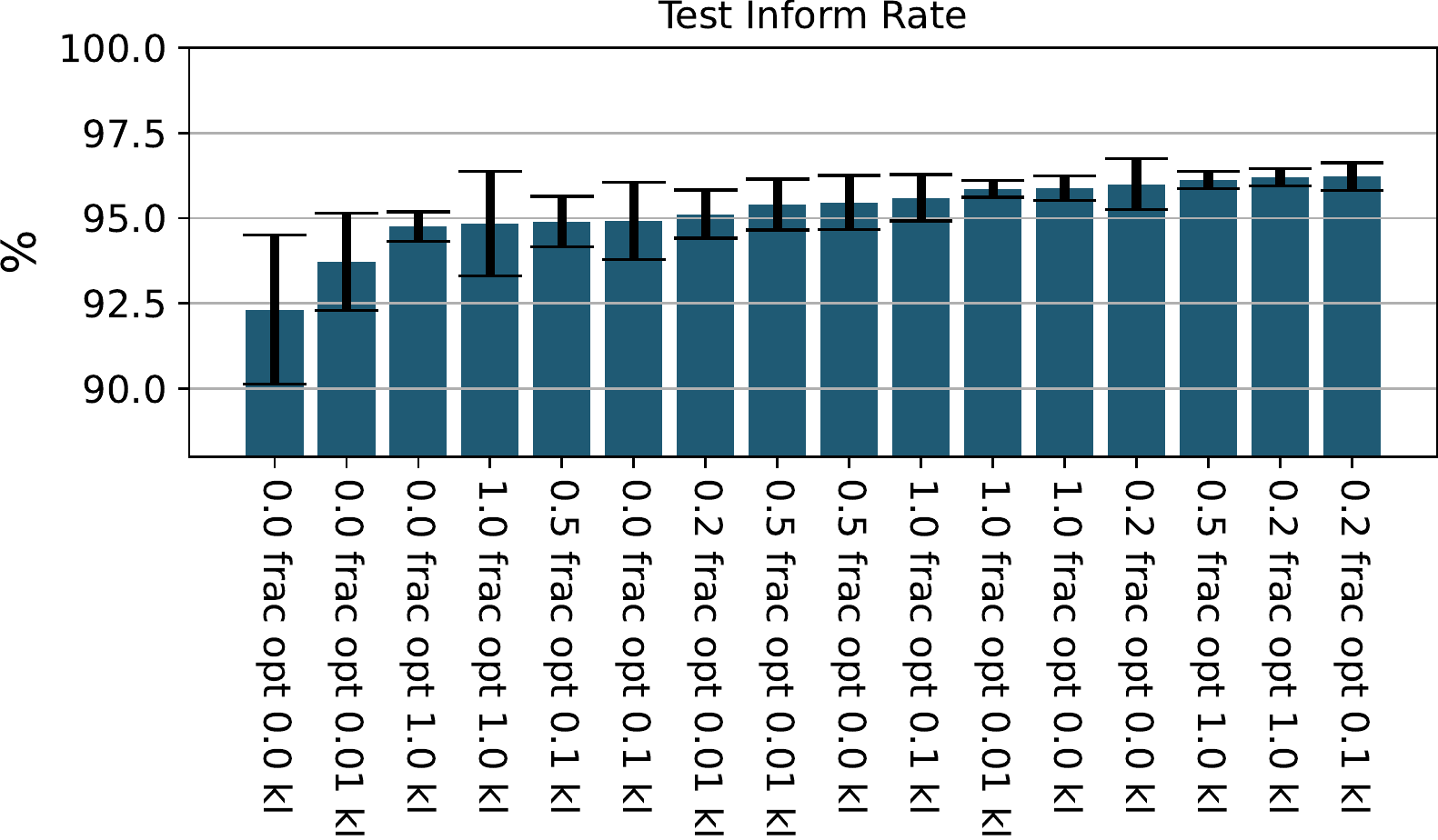}
        \caption{}
    \end{subfigure}
    \begin{subfigure}{\bsubfig}
        \includegraphics[width=\linewidth]{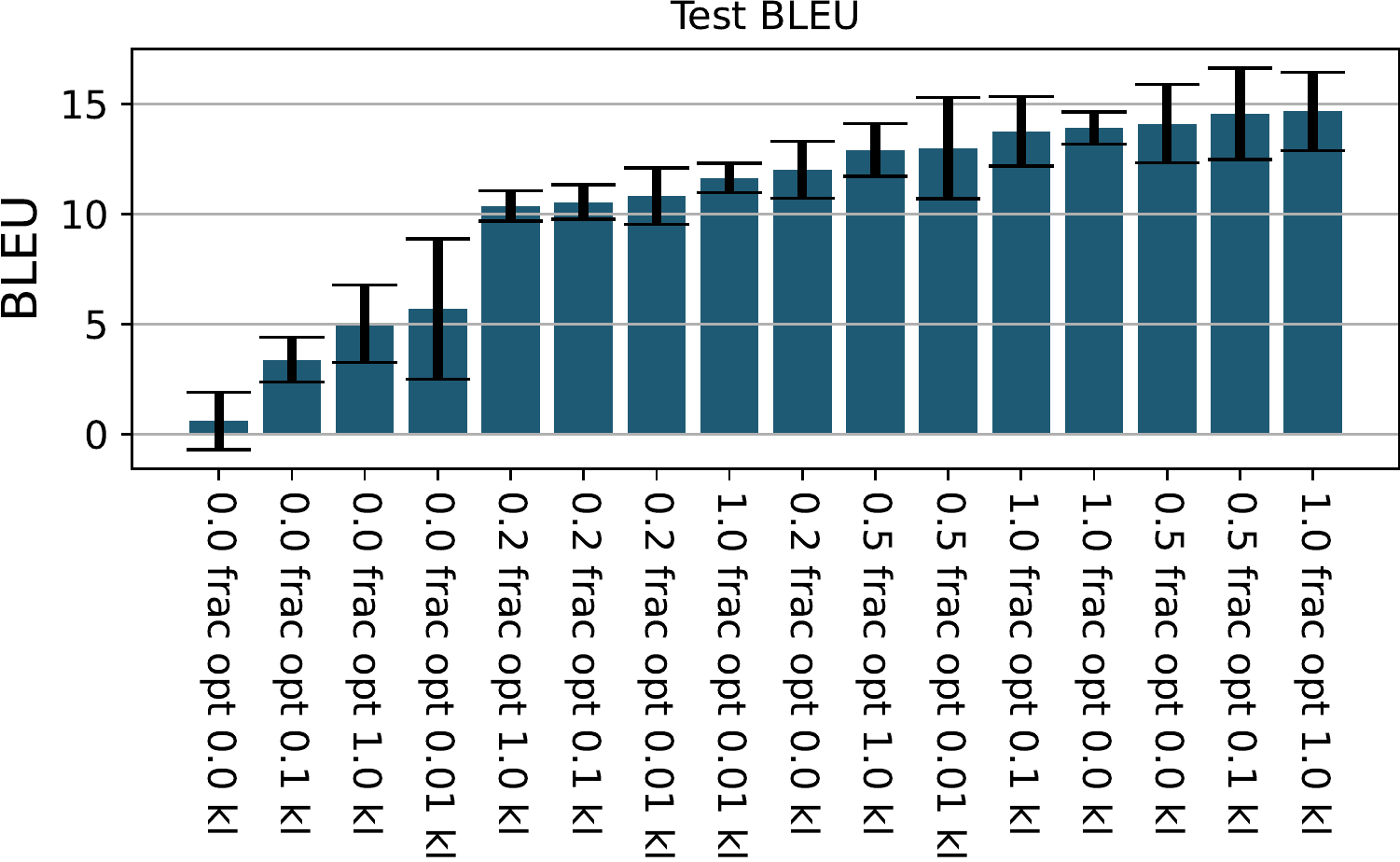}
        \caption{}
    \end{subfigure}
    \caption{Regularization sensitivity 400 epochs.}
    \label{app:fig:regularization-sensitivity}
\end{figure*}


\section{Model and Training Setup}\label{sec:app:modeltraining}
We make use of the Long-Short Term Memory \citep{hochreiter1997long} with a dot-product attention mechanism for the decoder architecture
and Gated Recurrent Unit \citep{cho2014learning} for the encoder architecture.
In both cases, we use a hidden size of $300$ and a latent embedding vector of size $200$.
Note that in the case of prior-posterior shuffling, there are shared parameters for $\p$ and $\qp$ through the encoder, the hidden state
is mapped to the appropriate mean and variance by independent neural networks.
Under the context $\rc$ we assume access to the ground-truth dialogue state, this can be easily replaced by the predictions of a dialogue state-tracker \citep{li2020coco,lee2021sumbt+}.

During training we choose the best models by keeping track of
validation success rate and BLEU score and evaluate the best performing checkpoint on the test set. As discussed, the training of \method{} is separated into two stages.

\begin{figure*}
    \setlength{\bsubfig}{{0.3\linewidth}}
    \begin{subfigure}{\bsubfig}
        \includegraphics[width=\linewidth]{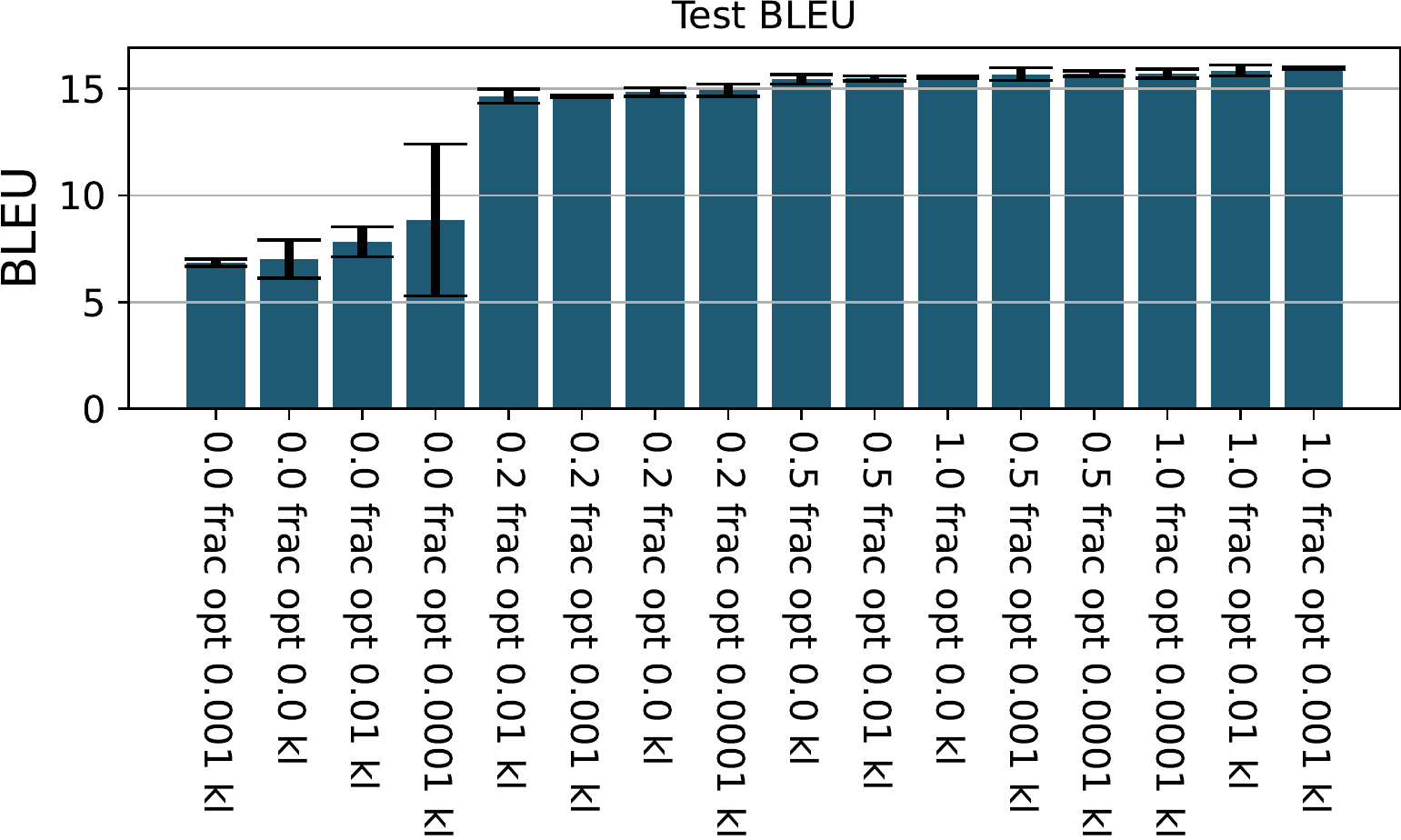}
        \caption{}
    \end{subfigure}
    \begin{subfigure}{\bsubfig}
        \includegraphics[width=\linewidth]{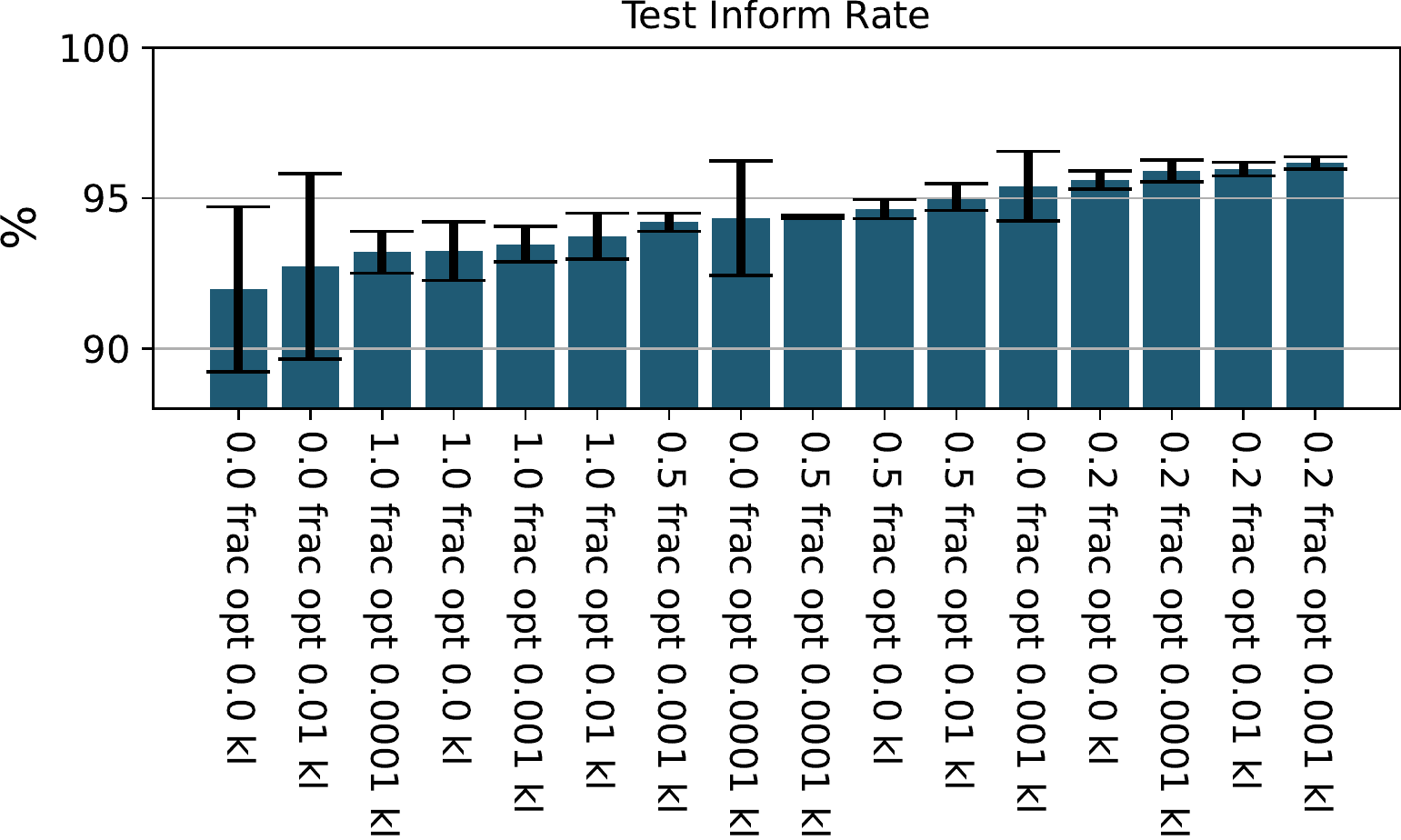}
        \caption{}
    \end{subfigure}
    \begin{subfigure}{\bsubfig}
        \includegraphics[width=\linewidth]{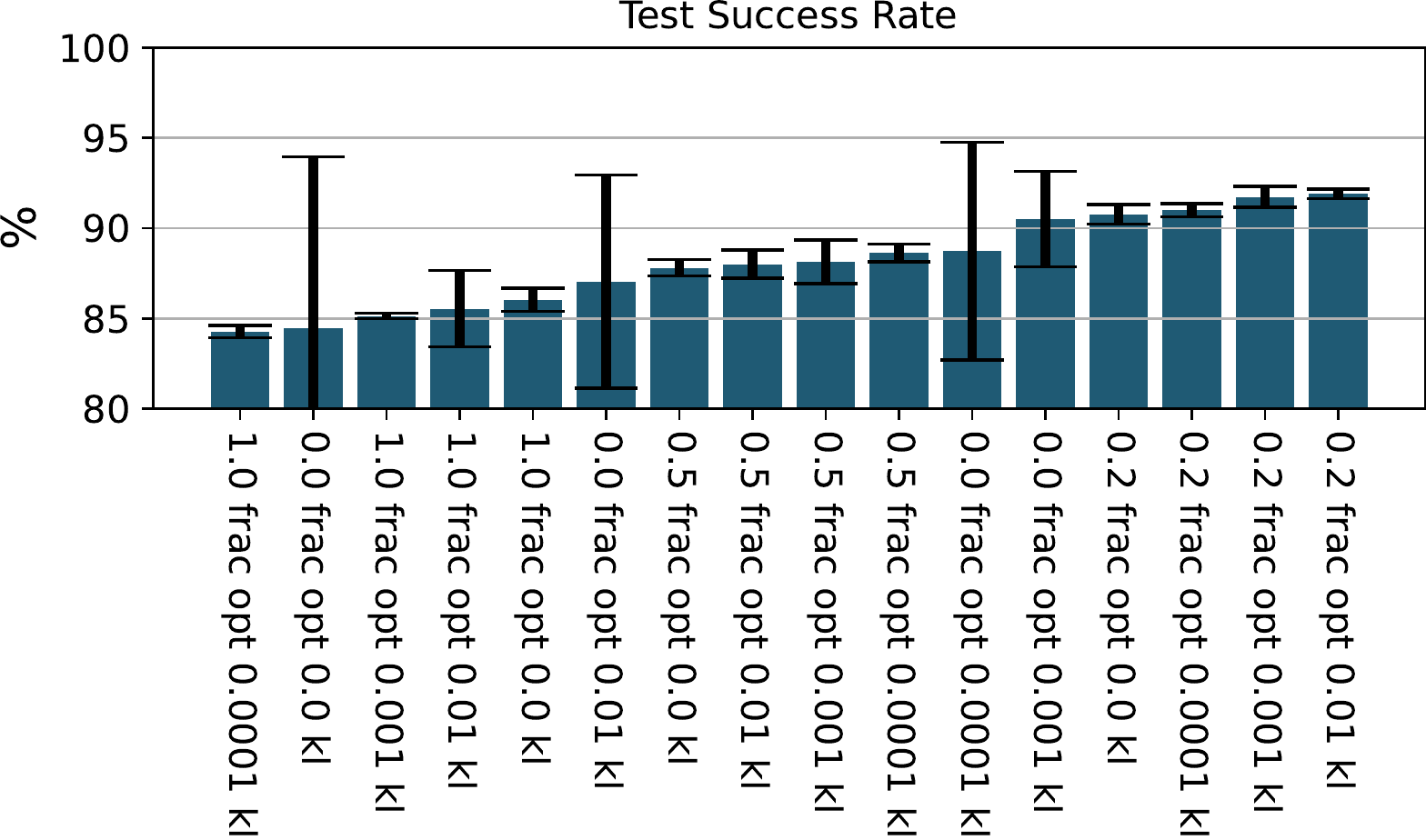}
        \caption{}
    \end{subfigure}
    \begin{subfigure}{\bsubfig}
        \includegraphics[width=\linewidth]{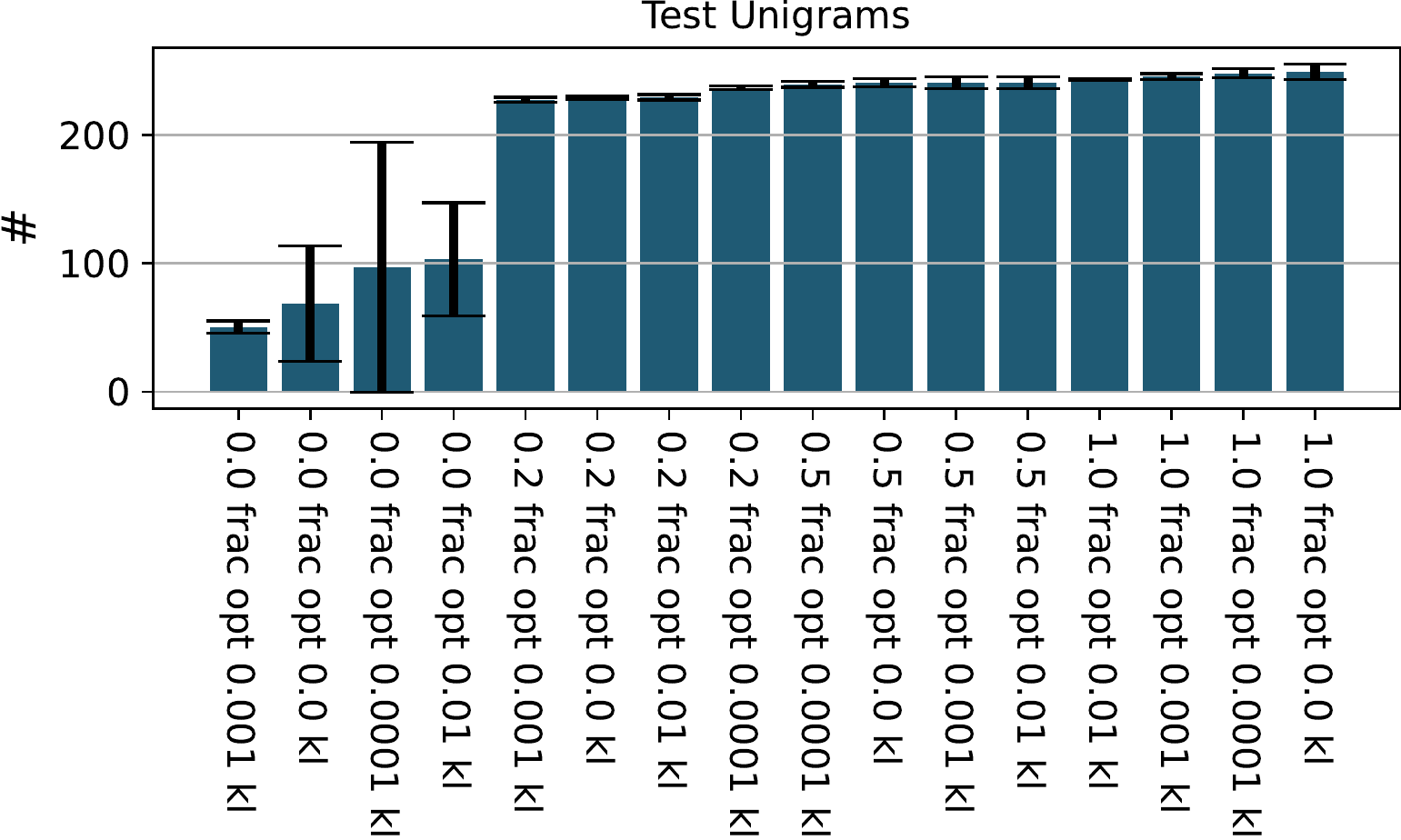}
        \caption{}
    \end{subfigure}
    \begin{subfigure}{\bsubfig}
        \includegraphics[width=\linewidth]{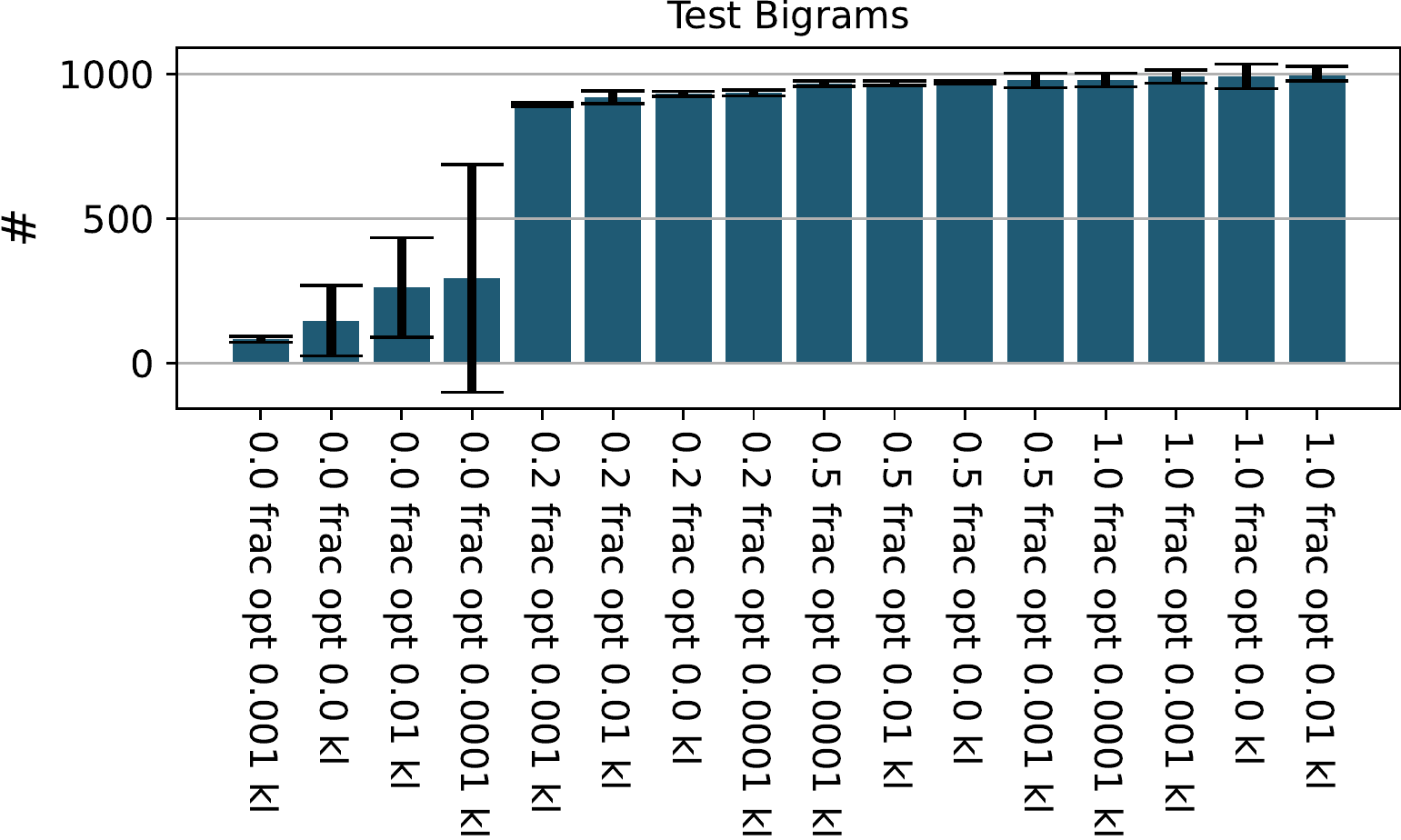}
        \caption{}
    \end{subfigure}
    \begin{subfigure}{\bsubfig}
        \includegraphics[width=\linewidth]{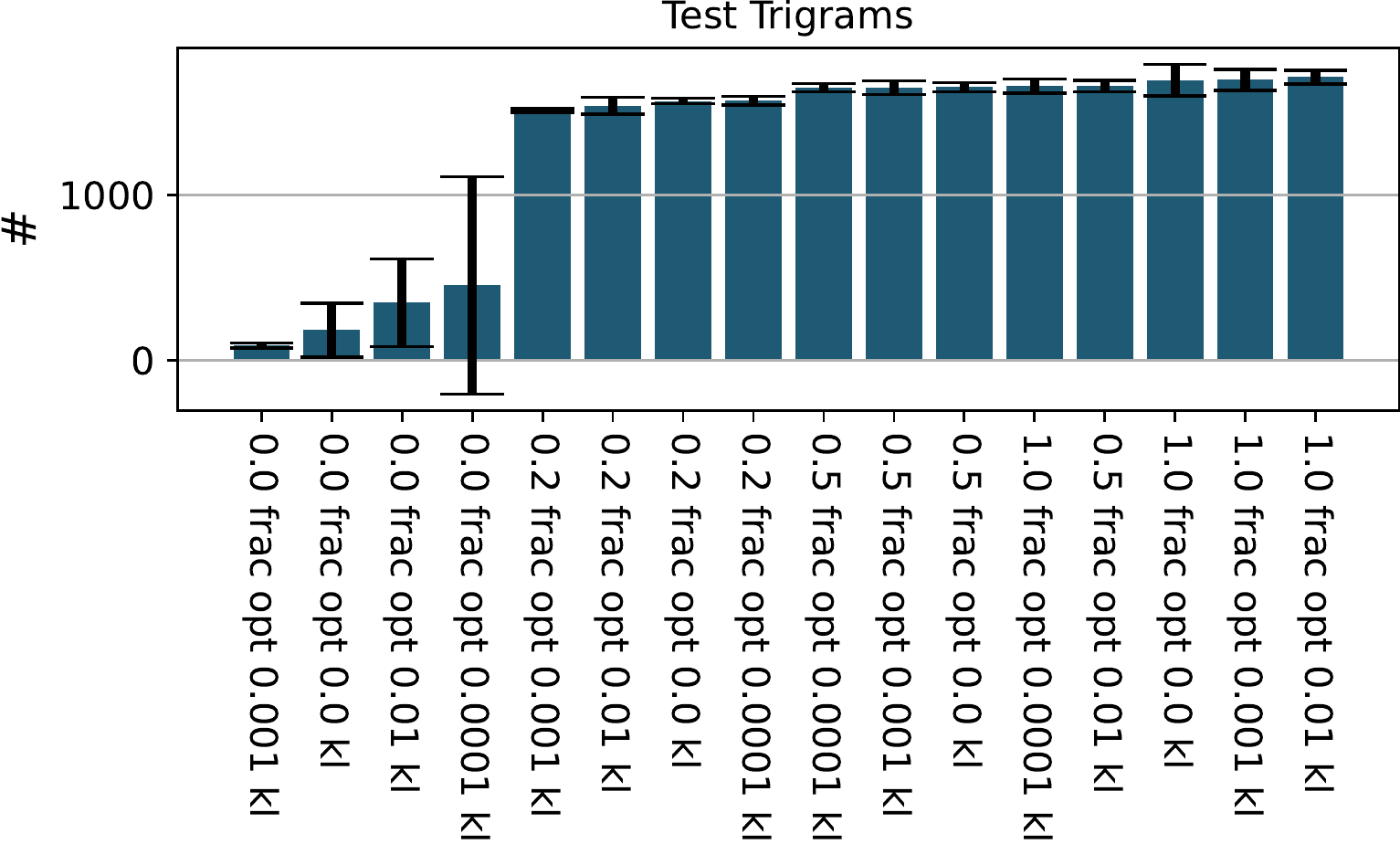}
        \caption{}
    \end{subfigure}
    \caption{Sensitivity analysis for the categorical latent case with 10 30-dimensional categoricals in the latent.}
    \label{app:fig:categorical-gs}
\end{figure*}






%% file: main.bbl
\begin{thebibliography}{40}
\expandafter\ifx\csname natexlab\endcsname\relax\def\natexlab#1{#1}\fi

\bibitem[{Arjovsky et~al.(2017)Arjovsky, Chintala, and
  Bottou}]{arjovsky2017wasserstein}
Martin Arjovsky, Soumith Chintala, and L{\'e}on Bottou. 2017.
\newblock Wasserstein generative adversarial networks.
\newblock In \emph{International conference on machine learning}, pages
  214--223. PMLR.

\bibitem[{Budzianowski et~al.(2018)Budzianowski, Wen, Tseng, Casanueva, Ultes,
  Ramadan, and Ga{\v{s}}i{\'c}}]{budzianowski2018MultiWOZ}
Pawe{\l} Budzianowski, Tsung-Hsien Wen, Bo-Hsiang Tseng, Inigo Casanueva,
  Stefan Ultes, Osman Ramadan, and Milica Ga{\v{s}}i{\'c}. 2018.
\newblock Multiwoz--a large-scale multi-domain wizard-of-oz dataset for
  task-oriented dialogue modelling.
\newblock \emph{arXiv preprint arXiv:1810.00278}.

\bibitem[{Cali{\'n}ski and Harabasz(1974)}]{calinski1974dendrite}
Tadeusz Cali{\'n}ski and Jerzy Harabasz. 1974.
\newblock A dendrite method for cluster analysis.
\newblock \emph{Communications in Statistics-theory and Methods}, 3(1):1--27.

\bibitem[{Chen et~al.(2019)Chen, Chen, Qin, Yan, and
  Wang}]{chen2019semantically}
Wenhu Chen, Jianshu Chen, Pengda Qin, Xifeng Yan, and William~Yang Wang. 2019.
\newblock Semantically conditioned dialog response generation via hierarchical
  disentangled self-attention.
\newblock \emph{arXiv preprint arXiv:1905.12866}.

\bibitem[{Cho et~al.(2014)Cho, van Merri{\"e}nboer, Gulcehre, Bahdanau,
  Bougares, Schwenk, and Bengio}]{cho2014learning}
Kyunghyun Cho, Bart van Merri{\"e}nboer, Caglar Gulcehre, Dzmitry Bahdanau,
  Fethi Bougares, Holger Schwenk, and Yoshua Bengio. 2014.
\newblock Learning phrase representations using rnn encoder--decoder for
  statistical machine translation.
\newblock In \emph{Proceedings of the 2014 Conference on Empirical Methods in
  Natural Language Processing (EMNLP)}, pages 1724--1734.

\bibitem[{Fatemi et~al.(2016)Fatemi, Asri, Schulz, He, and
  Suleman}]{fatemi2016policy}
Mehdi Fatemi, Layla~El Asri, Hannes Schulz, Jing He, and Kaheer Suleman. 2016.
\newblock Policy networks with two-stage training for dialogue systems.
\newblock \emph{arXiv preprint arXiv:1606.03152}.

\bibitem[{Gao et~al.(2018)Gao, Galley, and Li}]{gao2018neural}
Jianfeng Gao, Michel Galley, and Lihong Li. 2018.
\newblock Neural approaches to conversational ai.
\newblock In \emph{The 41st International ACM SIGIR Conference on Research \&
  Development in Information Retrieval}, pages 1371--1374.

\bibitem[{Gao et~al.(2019)Gao, Galley, and Li}]{Gao:2019}
Jianfeng Gao, Michel Galley, and Lihong Li. 2019.
\newblock now.
\newblock \href {https://ieeexplore.ieee.org/document/8649787} {[link]}.

\bibitem[{Gu et~al.(2018)Gu, Cho, Ha, and Kim}]{gu2018dialogwae}
Xiaodong Gu, Kyunghyun Cho, Jung-Woo Ha, and Sunghun Kim. 2018.
\newblock Dialogwae: Multimodal response generation with conditional
  wasserstein auto-encoder.
\newblock \emph{arXiv preprint arXiv:1805.12352}.

\bibitem[{Henderson et~al.(2008)Henderson, Lemon, and
  Georgila}]{henderson2008hybrid}
James Henderson, Oliver Lemon, and Kallirroi Georgila. 2008.
\newblock Hybrid reinforcement/supervised learning of dialogue policies from
  fixed data sets.
\newblock \emph{Computational Linguistics}, 34(4):487--511.

\bibitem[{Henderson et~al.(2019)Henderson, Vuli{\'c}, Gerz, Casanueva,
  Budzianowski, Coope, Spithourakis, Wen, Mrk{\v{s}}i{\'c}, and
  Su}]{henderson2019training}
Matthew Henderson, Ivan Vuli{\'c}, Daniela Gerz, I{\~n}igo Casanueva, Pawe{\l}
  Budzianowski, Sam Coope, Georgios Spithourakis, Tsung-Hsien Wen, Nikola
  Mrk{\v{s}}i{\'c}, and Pei-Hao Su. 2019.
\newblock Training neural response selection for task-oriented dialogue
  systems.
\newblock \emph{arXiv preprint arXiv:1906.01543}.

\bibitem[{Hochreiter and Schmidhuber(1997)}]{hochreiter1997long}
Sepp Hochreiter and J{\"u}rgen Schmidhuber. 1997.
\newblock Long short-term memory.
\newblock \emph{Neural computation}, 9(8):1735--1780.

\bibitem[{Inaba and Takahashi(2016)}]{inaba2016neural}
Michimasa Inaba and Kenichi Takahashi. 2016.
\newblock Neural utterance ranking model for conversational dialogue systems.
\newblock In \emph{Proceedings of the 17th Annual Meeting of the Special
  Interest Group on Discourse and Dialogue}, pages 393--403.

\bibitem[{Jiang et~al.(2021)Jiang, Dai, Yang, Zhao, and Wei}]{jiang2021towards}
Haoming Jiang, Bo~Dai, Mengjiao Yang, Tuo Zhao, and Wei Wei. 2021.
\newblock Towards automatic evaluation of dialog systems: A model-free
  off-policy evaluation approach.
\newblock \emph{arXiv preprint arXiv:2102.10242}.

\bibitem[{Kingma and Welling(2013)}]{kingma2013auto}
Diederik~P Kingma and Max Welling. 2013.
\newblock Auto-encoding variational bayes.
\newblock \emph{arXiv preprint arXiv:1312.6114}.

\bibitem[{Kottur et~al.(2017)Kottur, Moura, Lee, and Batra}]{kottur2017natural}
Satwik Kottur, Jos{\'e}~MF Moura, Stefan Lee, and Dhruv Batra. 2017.
\newblock Natural language does not emerge'naturally'in multi-agent dialog.
\newblock \emph{arXiv preprint arXiv:1706.08502}.

\bibitem[{Le et~al.(2020)Le, Sahoo, Liu, Chen, and Hoi}]{le2020uniconv}
Hung Le, Doyen Sahoo, Chenghao Liu, Nancy~F Chen, and Steven~CH Hoi. 2020.
\newblock Uniconv: A unified conversational neural architecture for
  multi-domain task-oriented dialogues.
\newblock \emph{arXiv preprint arXiv:2004.14307}.

\bibitem[{Lee et~al.(2021)Lee, Jo, Kim, Jung, and Kim}]{lee2021sumbt+}
Hwaran Lee, Seokhwan Jo, Hyungjun Kim, Sangkeun Jung, and Tae-Yoon Kim. 2021.
\newblock Sumbt+ larl: Effective multi-domain end-to-end neural task-oriented
  dialog system.
\newblock \emph{IEEE Access}, 9:116133--116146.

\bibitem[{Lewis et~al.(2017)Lewis, Yarats, Dauphin, Parikh, and
  Batra}]{lewis2017deal}
Mike Lewis, Denis Yarats, Yann~N Dauphin, Devi Parikh, and Dhruv Batra. 2017.
\newblock Deal or no deal? end-to-end learning for negotiation dialogues.
\newblock \emph{arXiv preprint arXiv:1706.05125}.

\bibitem[{LI et~al.(2020)LI, Yavuz, Hashimoto, Li, Niu, Rajani, Yan, Zhou, and
  Xiong}]{li2020coco}
SHIYANG LI, Semih Yavuz, Kazuma Hashimoto, Jia Li, Tong Niu, Nazneen Rajani,
  Xifeng Yan, Yingbo Zhou, and Caiming Xiong. 2020.
\newblock Coco: Controllable counterfactuals for evaluating dialogue state
  trackers.
\newblock In \emph{International Conference on Learning Representations}.

\bibitem[{Lubis et~al.(2020)Lubis, Geishauser, Heck, Lin, Moresi, van Niekerk,
  and Gasic}]{lubis:2020}
Nurul Lubis, Christian Geishauser, Michael Heck, Hsien{-}Chin Lin, Marco
  Moresi, Carel van Niekerk, and Milica Gasic. 2020.
\newblock \href {https://doi.org/10.18653/v1/2020.coling-main.41} {{LAVA:}
  latent action spaces via variational auto-encoding for dialogue policy
  optimization}.
\newblock In \emph{Proceedings of the 28th International Conference on
  Computational Linguistics, {COLING} 2020, Barcelona, Spain (Online), December
  8-13, 2020}, pages 465--479. International Committee on Computational
  Linguistics.

\bibitem[{McInnes et~al.(2018)McInnes, Healy, and Melville}]{mcinnes2018umap}
Leland McInnes, John Healy, and James Melville. 2018.
\newblock Umap: Uniform manifold approximation and projection for dimension
  reduction.
\newblock \emph{arXiv preprint arXiv:1802.03426}.

\bibitem[{Mehri and Eskenazi(2020)}]{mehri2020unsupervised}
Shikib Mehri and Maxine Eskenazi. 2020.
\newblock Unsupervised evaluation of interactive dialog with dialogpt.
\newblock \emph{arXiv preprint arXiv:2006.12719}.

\bibitem[{Mehri et~al.(2019)Mehri, Srinivasan, and
  Eskenazi}]{mehri2019structured}
Shikib Mehri, Tejas Srinivasan, and Maxine Eskenazi. 2019.
\newblock Structured fusion networks for dialog.
\newblock \emph{arXiv preprint arXiv:1907.10016}.

\bibitem[{Mnih et~al.(2013)Mnih, Kavukcuoglu, Silver, Graves, Antonoglou,
  Wierstra, and Riedmiller}]{mnih2013playing}
Volodymyr Mnih, Koray Kavukcuoglu, David Silver, Alex Graves, Ioannis
  Antonoglou, Daan Wierstra, and Martin Riedmiller. 2013.
\newblock Playing atari with deep reinforcement learning.
\newblock \emph{arXiv preprint arXiv:1312.5602}.

\bibitem[{Park et~al.(2018)Park, Cho, and Kim}]{park2018hierarchical}
Yookoon Park, Jaemin Cho, and Gunhee Kim. 2018.
\newblock A hierarchical latent structure for variational conversation
  modeling.
\newblock \emph{arXiv preprint arXiv:1804.03424}.

\bibitem[{Saleh et~al.(2020)Saleh, Jaques, Ghandeharioun, Shen, and
  Picard}]{saleh2020hierarchical}
Abdelrhman Saleh, Natasha Jaques, Asma Ghandeharioun, Judy Shen, and Rosalind
  Picard. 2020.
\newblock Hierarchical reinforcement learning for open-domain dialog.
\newblock In \emph{Proceedings of the AAAI Conference on Artificial
  Intelligence}, volume~34, pages 8741--8748.

\bibitem[{Serban et~al.(2017)Serban, Sordoni, Lowe, Charlin, Pineau, Courville,
  and Bengio}]{serban2017hierarchical}
Iulian Serban, Alessandro Sordoni, Ryan Lowe, Laurent Charlin, Joelle Pineau,
  Aaron Courville, and Yoshua Bengio. 2017.
\newblock A hierarchical latent variable encoder-decoder model for generating
  dialogues.
\newblock In \emph{Proceedings of the AAAI Conference on Artificial
  Intelligence}, volume~31.

\bibitem[{Shen et~al.(2018)Shen, Su, Niu, and Demberg}]{shen2018improving}
Xiaoyu Shen, Hui Su, Shuzi Niu, and Vera Demberg. 2018.
\newblock Improving variational encoder-decoders in dialogue generation.
\newblock In \emph{Proceedings of the AAAI conference on artificial
  intelligence}, volume~32.

\bibitem[{Shin et~al.(2019)Shin, Xu, Madotto, and Fung}]{shin2019happybot}
Jamin Shin, Peng Xu, Andrea Madotto, and Pascale Fung. 2019.
\newblock Happybot: Generating empathetic dialogue responses by improving user
  experience look-ahead.
\newblock \emph{arXiv preprint arXiv:1906.08487}.

\bibitem[{Su et~al.(2017)Su, Budzianowski, Ultes, Gasic, and
  Young}]{su2017sample}
Pei-Hao Su, Pawel Budzianowski, Stefan Ultes, Milica Gasic, and Steve Young.
  2017.
\newblock Sample-efficient actor-critic reinforcement learning with supervised
  data for dialogue management.
\newblock \emph{arXiv preprint arXiv:1707.00130}.

\bibitem[{Tao et~al.(2019)Tao, Wu, Xu, Hu, Zhao, and Yan}]{tao2019multi}
Chongyang Tao, Wei Wu, Can Xu, Wenpeng Hu, Dongyan Zhao, and Rui Yan. 2019.
\newblock Multi-representation fusion network for multi-turn response selection
  in retrieval-based chatbots.
\newblock In \emph{Proceedings of the twelfth ACM international conference on
  web search and data mining}, pages 267--275.

\bibitem[{Walker et~al.(2007)Walker, Stent, Mairesse, and
  Prasad}]{walker2007individual}
Marilyn~A Walker, Amanda Stent, Fran{\c{c}}ois Mairesse, and Rashmi Prasad.
  2007.
\newblock Individual and domain adaptation in sentence planning for dialogue.
\newblock \emph{Journal of Artificial Intelligence Research}, 30:413--456.

\bibitem[{Wang et~al.(2020)Wang, Tian, Wang, Quan, and Yu}]{wang2020multi}
Kai Wang, Junfeng Tian, Rui Wang, Xiaojun Quan, and Jianxing Yu. 2020.
\newblock Multi-domain dialogue acts and response co-generation.
\newblock \emph{arXiv preprint arXiv:2004.12363}.

\bibitem[{Williams et~al.(2017)Williams, Asadi, and Zweig}]{williams2017hybrid}
Jason~D Williams, Kavosh Asadi, and Geoffrey Zweig. 2017.
\newblock Hybrid code networks: practical and efficient end-to-end dialog
  control with supervised and reinforcement learning.
\newblock \emph{arXiv preprint arXiv:1702.03274}.

\bibitem[{Xu et~al.(2018)Xu, Wu, and Wu}]{xu2018towards}
Can Xu, Wei Wu, and Yu~Wu. 2018.
\newblock Towards explainable and controllable open domain dialogue generation
  with dialogue acts.
\newblock \emph{arXiv preprint arXiv:1807.07255}.

\bibitem[{Yan et~al.(2016)Yan, Song, and Wu}]{yan2016learning}
Rui Yan, Yiping Song, and Hua Wu. 2016.
\newblock Learning to respond with deep neural networks for retrieval-based
  human-computer conversation system.
\newblock In \emph{Proceedings of the 39th International ACM SIGIR conference
  on Research and Development in Information Retrieval}, pages 55--64.

\bibitem[{Zhang et~al.(2020)Zhang, Ou, and Yu}]{zhang2020task}
Yichi Zhang, Zhijian Ou, and Zhou Yu. 2020.
\newblock Task-oriented dialog systems that consider multiple appropriate
  responses under the same context.
\newblock In \emph{Proceedings of the AAAI Conference on Artificial
  Intelligence}, volume~34, pages 9604--9611.

\bibitem[{Zhao et~al.(2019)Zhao, Xie, and Eskenazi}]{zhao:2019}
Tiancheng Zhao, Kaige Xie, and Maxine Eskenazi. 2019.
\newblock \href {https://doi.org/10.18653/v1/N19-1123} {Rethinking action
  spaces for reinforcement learning in end-to-end dialog agents with latent
  variable models}.
\newblock In \emph{Proceedings of the 2019 Conference of the North {A}merican
  Chapter of the Association for Computational Linguistics: Human Language
  Technologies, Volume 1 (Long and Short Papers)}, pages 1208--1218,
  Minneapolis, Minnesota. Association for Computational Linguistics.

\bibitem[{Zhao et~al.(2017)Zhao, Zhao, and Eskenazi}]{zhao2017learning}
Tiancheng Zhao, Ran Zhao, and Maxine Eskenazi. 2017.
\newblock Learning discourse-level diversity for neural dialog models using
  conditional variational autoencoders.
\newblock \emph{arXiv preprint arXiv:1703.10960}.

\end{thebibliography}
